\documentclass[sigconf]{acmart}
\renewcommand\footnotetextcopyrightpermission[1]{} 

\usepackage{amsmath}

\usepackage{newtxmath}
\usepackage{amsthm}
\usepackage{bm}
\usepackage{subcaption}
\usepackage{algorithm}
\usepackage{algorithmic}
\usepackage{caption} 
\DeclareMathOperator*{\argmax}{arg\,max}

\newcommand{\vecx}{\mathbf{x}}
\newcommand{\vecy}{\mathbf{y}}
\newcommand{\matK}{\mathbf{K}}

\newcommand{\Ex}{\mathbb{E}}

\newtheorem{theorem}{Theorem}

\newtheorem{lemma}{Lemma}

\AtBeginDocument{%
  \providecommand\BibTeX{{%
    \normalfont B\kern-0.5em{\scshape i\kern-0.25em b}\kern-0.8em\TeX}}}

\begin{document}

\title[UnRAvEL-Select Wisely and Explain]{Select Wisely and Explain: Active Learning and Probabilistic Local Post-hoc Explainability}

\author{Aditya Saini}
\affiliation{%
  \institution{IIITD}
  \city{New Delhi}
  \country{India}}
\email{aditya18125@iiitd.ac.in}

\author{Ranjitha Prasad}
\affiliation{%
  \institution{IIITD}
  \city{New Delhi}
  \country{India}}
\email{ranjitha@iiitd.ac.in}

\begin{abstract}
Albeit the tremendous performance improvements in designing complex artificial intelligence (AI) systems in data-intensive domains, the black-box nature of these systems leads to the lack of trustworthiness. Post-hoc interpretability methods explain the prediction of a black-box ML model for a single instance, and such
explanations are being leveraged by domain experts to diagnose the underlying biases of these models. Despite their efficacy in providing valuable insights, existing approaches fail to deliver consistent and reliable explanations. In this paper, we propose an active learning-based technique called UnRAvEL (Uncertainty driven Robust Active Learning Based Locally Faithful Explanations), which consists of a novel acquisition function that is locally faithful and uses uncertainty-driven sampling based on the posterior distribution on the probabilistic locality using Gaussian process regression (GPR). We present a theoretical analysis of UnRAvEL by treating it as a local optimizer and analyzing its regret in terms of instantaneous regrets over a global optimizer. We demonstrate the efficacy of the local samples generated by UnRAvEL by incorporating different kernels such as the Matern and linear kernels in GPR. Through a series of experiments, we show that UnRAvEL outperforms the
baselines with respect to stability and local fidelity on several real-world models and datasets. We show that UnRAvEL is an efficient surrogate dataset generator by deriving importance scores on this surrogate dataset using sparse linear models. We also showcase the sample efficiency and flexibility of the developed framework on the Imagenet dataset using a pre-trained ResNet model.
\end{abstract}




\maketitle
\pagestyle{plain} 

\begin{figure*}
    \centering
    \begin{subfigure}[t]{0.3\linewidth}
    \includegraphics[width=\textwidth]{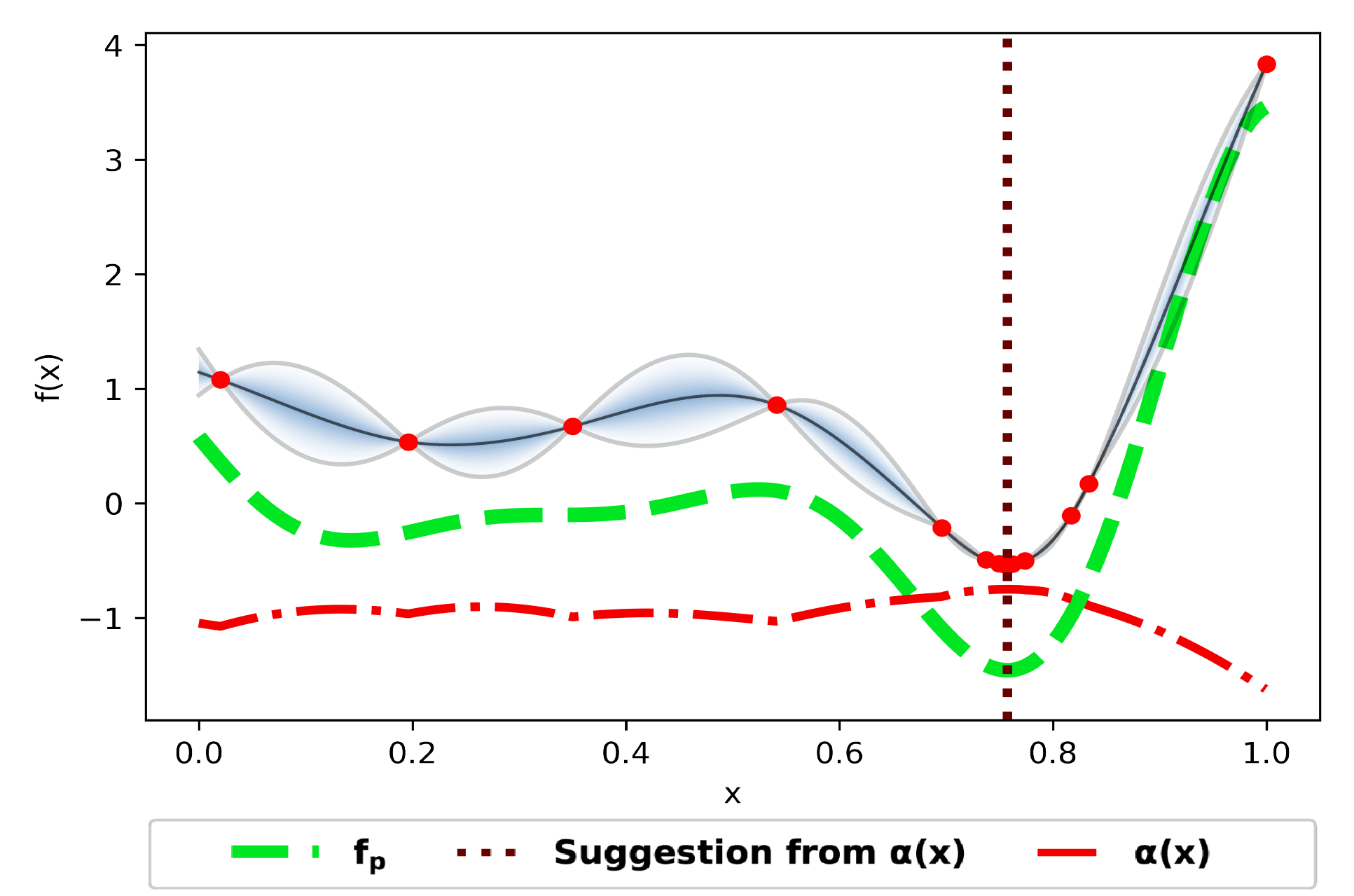}
        \caption{UCB(Upper confidence bound)}
  \end{subfigure}
    \begin{subfigure}[t]{0.3\linewidth}
    \centering    \includegraphics[width=\textwidth]{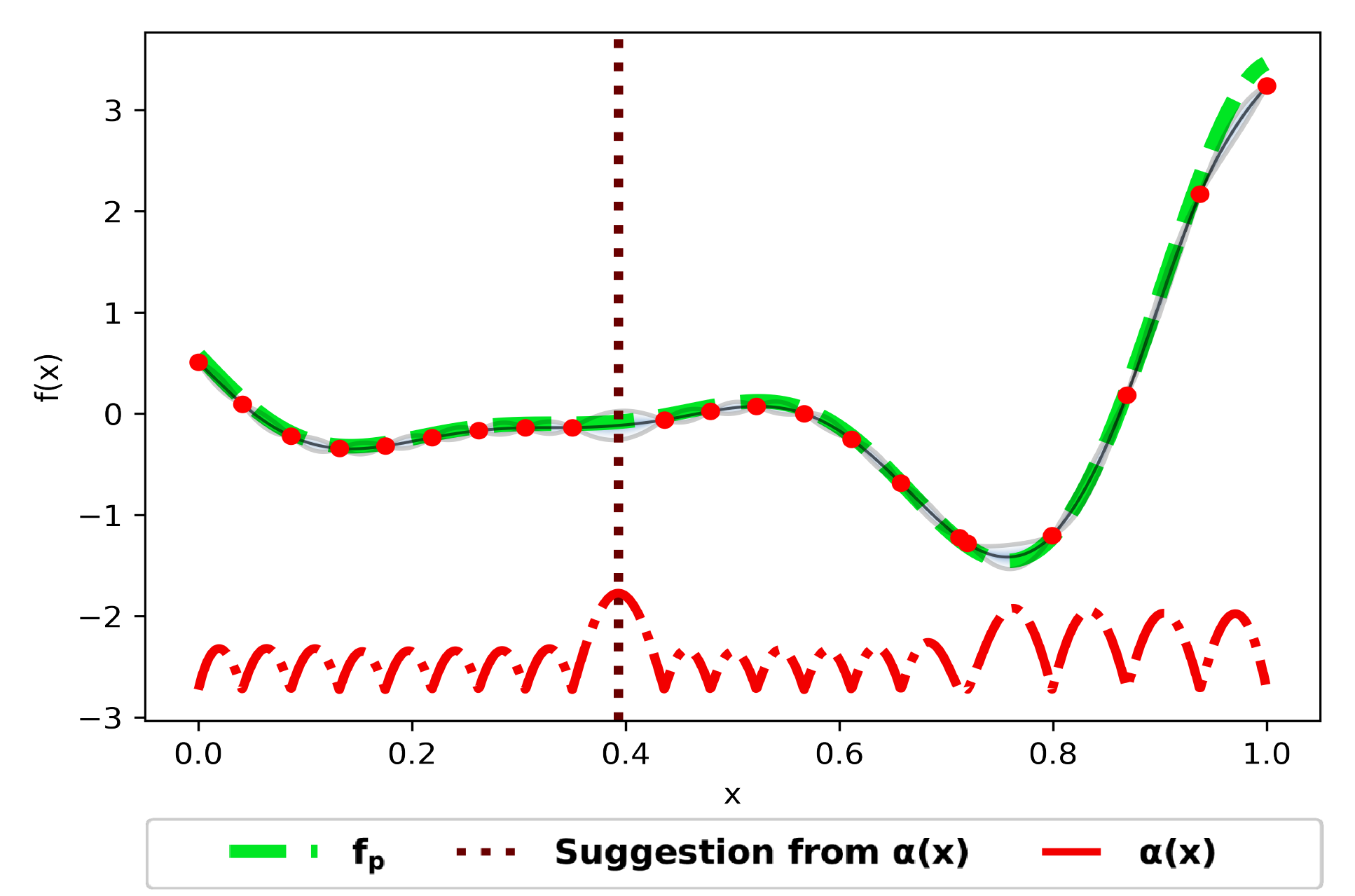}
    \caption{UR(Uncertainty Reduction)}
  \end{subfigure}
  \begin{subfigure}[t]{0.3\linewidth}
    \centering    \includegraphics[width=\textwidth]{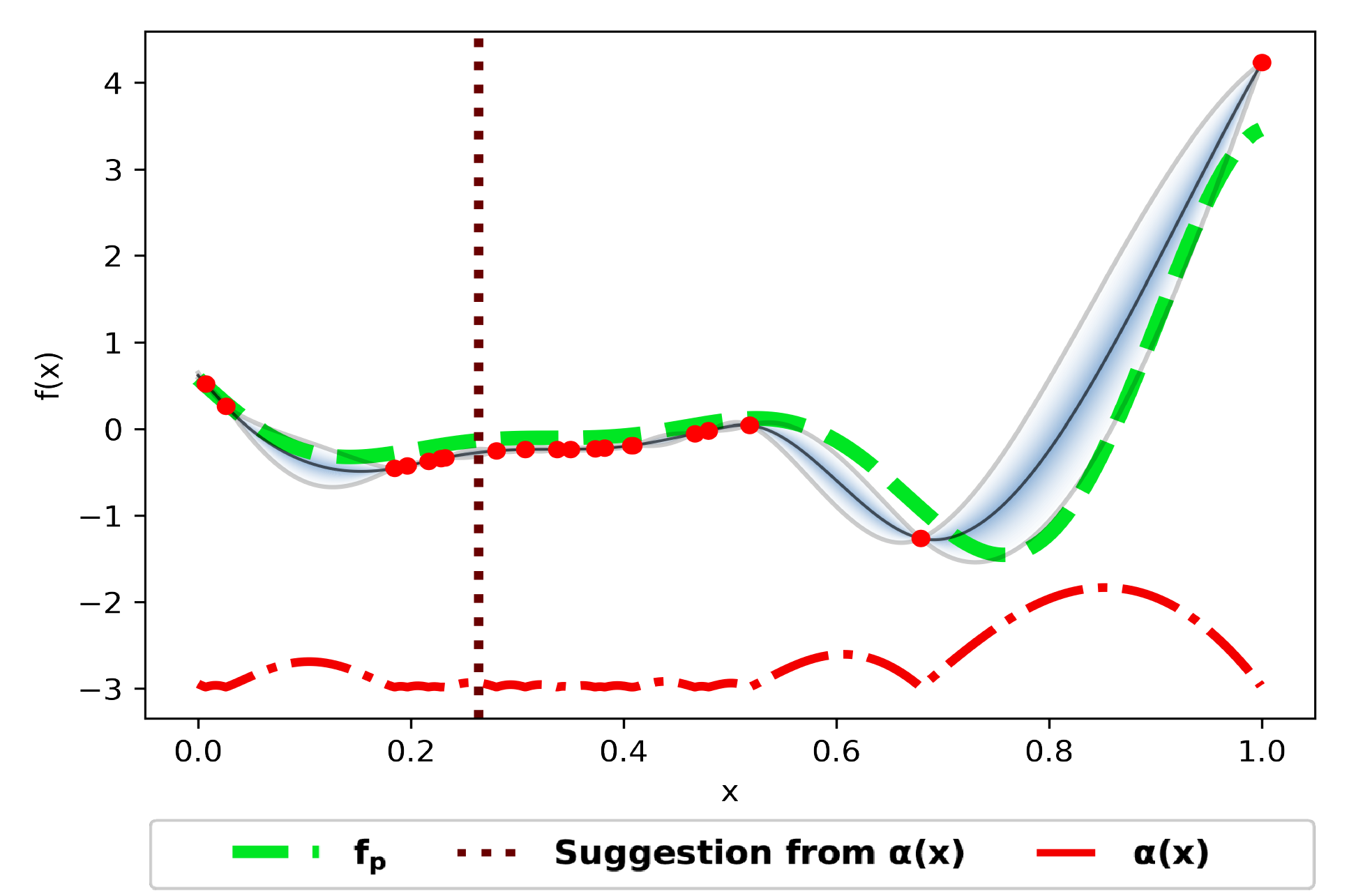}
    \caption{FUR(Faithful Uncertainty Reduction)}
  \end{subfigure}
    \caption{Optimizing the one-dimensional Forrester function\cite{forrester}: Comparison of the convergent iteration of UCB, UR and the proposed FUR acquisition functions}
    \label{fig:alpha_comparison_plot_intro}
    
\end{figure*}

\section{Introduction}

Modern-day applications generate staggering amounts of data, and extracting insightful information from this data is a challenging task. Recent advances in machine learning and deep learning promise advancements towards autonomous systems that will perceive, learn, decide, and act without the need for any human intervention. However, one of the most significant drawbacks of current models is their inability to explain their decisions and actions to humans. Recent studies\cite{social_attack, OnepixelFooling, healthcare_attack} have demonstrated that when deep neural networks are subject to adversaries, severe biases are induced in their decision-making process, hindering their prowess in sociological and healthcare-related domains. Unsurprisingly, several government authorities worldwide have mandated the need to provide explanations on predictions to end customers \cite{adadi2018peeking}. These aspects motivate research on explainable AI (XAI), notably feature attribution methods. Recently, several techniques have been proposed in the context of explainable AI. These methods can be classified as being local \cite{LIME, SHAP, ALIME, DLIME, OptiLIME} or global \cite{SHAP}, model agnostic \cite{LIME, SHAP} or model specific \cite{Deeplift}, intrinsic or post-hoc based, perturbation or saliency-based \cite{ALIME}, concept-based or feature-based\citep{adadi2018peeking}, etc.

Post-hoc, perturbation-based techniques like LIME\cite{LIME} and KernelSHAP \cite{SHAP} are some of the most popular methods for providing model agnostic explanations in the locality of a given instance (which we refer to as the \emph{index sample}). By producing weighted perturbations (surrogate data) in the neighborhood of the instance of interest, these techniques employ locally weighted methods to obtain per-feature importance weights. Despite the widespread usage of these techniques, subsequent works have pointed various issues. For instance, LIME leads to inconsistent explanations on a given sample \cite{Indices, BayLIME, DLIME, ALIME, OptiLIME}, hampers its use in safety-critical systems. Although KernelSHAP successfully counters the stability issue, it employs training data for explanations. In addition, its vulnerabilities to adversarial attacks \cite{FoolingLIME} can attract malicious counterfeit explainer attacks. Essentially, carefully crafted prediction models can hide their biases by exploiting the vulnerabilities inherent in the out-of-distribution surrogate datasets generated by these local techniques. 

An insight derived from the literature\cite{XXAI} is that the lack of structure in the sampling process hampers the quality of surrogate data. In the extreme case, the data tends to be out of distribution, leading to information gain towards adversarial unsteady signals, which produce inconsistent explanations across iterations. This issue was addressed using focused sampling strategies \cite{Slack2020ReliablePH}, which smartly select high information samples from the surrogate dataset and use a linear model for explanations. The inductive bias in linear explainer models limits their capacity in guiding the sample selection process. In addition, LIME, KernelSHAP weigh the random samples, introducing additional hyperparameters that are very hard to choose optimally via cross-validation\cite{OptiLIME, BayLIME}. 

In this work, we propose UnRAvEL(Uncertainty driven Robust Active Learning Based Locally Faithful Explanations), which incorporates Active Learning based sampling and Gaussian process regression (GPR) for generating informative surrogate datasets, which can be combined with several simple models for obtaining explanations. Driving inspiration from active learning strategies like uncertainty sampling \cite{UR} and the exploration-exploitation structure of the Bayesian optimization technique UCB (Upper Confidence Bound) \cite{UCBsrinivas2010gaussian}, we propose a novel acquisition function called \textbf{FUR(Faithful Uncertainty Reduction)}. UnRAvEL intertwines sampling with GPR, which leads to a greedy trade-off between local fidelity via FUR and information gain via GPR. This is illustrated in Fig.~\ref{fig:alpha_comparison_plot_intro}, where it can be seen that at a later iteration, FUR greedily samples close to the index sample, whereas UR and UCB do not show any affinity to the index sample. UnRAvEL interprets locality in a probabilistic sense, and sampling a surrogate dataset is based on the posterior density obtained from GPR. After obtaining the surrogate dataset, different approaches are considered for getting feature importance scores. An ARD explainer \citep{ARDpaananen} is employed for probabilistic explanations using the inverse length parameters of the covariance function, while UnRAvEL-LIME is proposed where a sparse linear model obtains the feature importance weights. We theoretically analyze FUR, which we consider a local optimizer, and explore the regret behavior compared to a global optimizer \cite{kim2020local}. A summarized workflow of UnRAvEL is depicted in Fig.~\ref{fig:main_workflow}.

We demonstrate the performance of UnRAvEL on several real-world classification and regression datasets. We show that UnRAvEL outperforms the benchmarks such as LIME and BayLIME in terms of stability, measured using the Jaccard distance. We also show that faithful explanations using the surrogate dataset in UnRAvEL are achieved using small number of surrogate data samples. We demonstrate the performance of UnRAvEL at low sample efficiency and compare its performance with LIME. Specifically, we consider the  Imagenet\citep{imagenet} dataset and the pre-trained ResNet-18\citep{resnet} model and compare its explanations with LIME and the popular image-based feature/importance attribution technique known as  GradCAM \cite{GradCAM}.

\begin{figure*}
    \centering
     \includegraphics[width=0.9\textwidth]{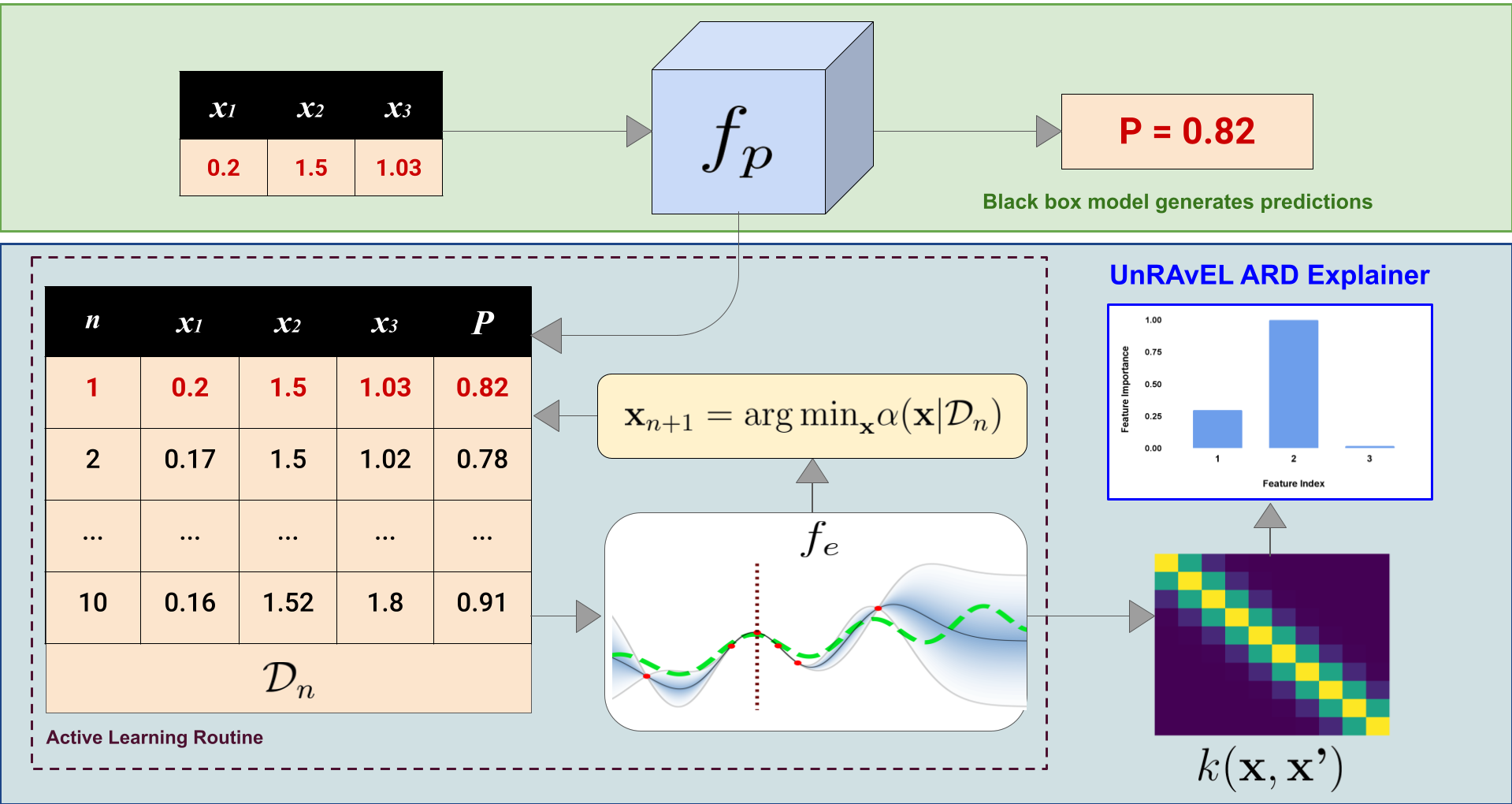}
    \caption{Workflow of UnRAvEL: Using the index sample (in red), the active learning module maximizes $\alpha(\vecx|\mathcal{D}_n)$ to  generate the surrogate dataset $\mathcal{D}_n$. The importance scores can be obtained directly from the kernel function of the GPR (ARD-explainer) or by using a linear explainer on the surrogate data (UnRAvEL-LIME). }
    \label{fig:main_workflow}
\end{figure*}

\section{Novelty and significance}

In this work, we are interested in perturbation-based model-agnostic post-hoc locally interpretable models. Among the well-known methods, LIME \citep{LIME} is one of the earliest popular approaches. We discuss the related works by highlighting the critical aspects of these works and discussing the impact of these aspects.\\
\textbf{Instability}: Several works have reported issues in LIME such as instability or inconsistency in the  explanations over several iterations \citep{DLIME,ALIME}. It is well-accepted that this inconsistency occurs due to random perturbation-based surrogate datasets. A deterministic hierarchical clustering approach for consistent explanations was proposed in DLIME\citep{DLIME}. DLIME requires the training data for clustering. Furthermore, the unequal distribution of points across clusters affected the fidelity of explanations. To avoid the additional task of 'explaining the explainer,' methods like ALIME\citep{ALIME,SmartSamplingGAN} to reduce instability are not preferred. A slight improvement to LIME is the parametric Bayesian method as proposed by \cite{BayLIME}, where a weighted sum of the prior knowledge and the estimates based on new samples obtained from LIME is used to get explanations in a Bayesian linear regression framework. Both LIME and BayLIME employ hyperparameters (kernel-width) that need to be tuned for sampling. In \cite{OptiLIME}, there is the freedom to choose the best adherence-stability trade-off level allowing the practitioner to decide whether the explanation is reliable, necessitating manual intervention. Recently, \cite{Slack2020ReliablePH} proposed a technique known as focused sampling, which utilizes uncertainty estimates to query the black box. \\
\textbf{Bayesian Methods and Uncertainty:} Few recent works have adopted a Bayesian formulation to explain black-box models. In \cite{guo2018explaining}, a Bayesian non-parametric approach was introduced to fit a global Dirichlet mixture across instances. BayLIME \cite{BayLIME} incorporates informative priors to improve the stability of the resulting explanations. In \cite{Slack2020ReliablePH}, the authors focus on modeling the uncertainty of local explanations to sample better from a randomly sampled dataset. In contrast, FUR sequentially selects every sample in the surrogate dataset based on the information gain. To maintain local fidelity, the notion of locality is incorporated into the acquisition function directly. Hence, there is no need for additional processing based on the kernel width. \\
\textbf{Adversarial Robustness:} Recent literature \cite{FoolingLIME, Fooling_geometry} has shown the danger of carefully crafted adversarial classifiers that target the out of distribution surrogate data to hide the inherent biases present in their decision-making process. The assumption is that the attacker can recreate the surrogate data and over-fit these samples to unbiased targets. As demonstrated in \cite{FoolingLIME}, perturbations methods like LIME and KernelSHAP are prone to such attacks because of their perturbation-driven data sampling routine. Moreover, authors in \cite{XXAI} have also theorized that the surrogate data generation pipeline by SHAP and LIME, though simplistic, is not interpretable because of the distance scaling process. To counter that, we motivate a need for a new sampling procedure that is information-theoretic driven and incorporates a distance calculating strategy that is both interpretable and free of hyperparameters.\\
\textbf{Sample Complexity:} Sample efficiency in post-hoc models is a crucial factor in efficiently obtaining reliable explanations, and there is consensus in the community that explainable models must use as few samples for an explanation as possible \cite{Slack2020ReliablePH}. Approaches such as LIME, KernelSHAP, and BayLIME do not provide any guidance on choosing the number of perturbations, although this issue has been acknowledged \cite{BayLIME}. Using FUR, we provide a framework that converges and ceases sampling as the information gain increases and the surrogate samples get closer to the index sample.

\section{Background and Notation}
LIME, KernelSHAP, BayLIME, and BayesSHAP/BayesLIME \citep{Slack2020ReliablePH} are among the popular model-agnostic, local explanation approaches designed to explain any black-box ML model. These methods explain individual predictions of any classifier/regressor in an interpretable manner by learning an easy-to-interpret ML model locally around each prediction. Each of these methods estimates feature attributions on individual instances, which capture the contribution of each feature on the black box prediction.

Let the sample to be explained be $\vecx_0$, which we refer to as the \emph{index sample}. Let $f_p$ denote the black-box model that takes a data point $\vecx_i$ as input and returns a target, i.e., $y_i = f_p(\vecx_i)$. The goal of local post-hoc  model-agnostic methods is to explain $f_p$ in an interpretable and faithful manner without assuming any knowledge about the internal workings of $f_p$. We build an explainer model $f_e(\cdot)$ which employs a surrogate dataset to obtain feature attributions. LIME and KernelSHAP employ sparse linear models, while BayLIME and BayesLIME/BayesSHAP employ Bayesian linear modeling for explaining the black-box model. Mathematically, LIME and KernelSHAP consider a class of models $\mathcal{G}$, and a locally weighted square loss $L(\cdot)$, such that
\begin{align}
    L(f_p,f_e,\pi_{\vecx}) = \sum_{\vecx_0,\vecx \in \mathcal{X}} \pi_{\vecx}(\vecx_0)(f_p(\vecx_0) - f_e(\vecx))^2,
    \label{eq:basicLIME}
\end{align}
 where $f_e(\cdot) \in \mathcal{G}$, $\vecx_0$ is the index sample, and $\vecx \in \mathbb{R}^d$ is the perturbed sample in the original or sparse representation. LIME and KernelSHAP use different strategies to choose $\pi_{\vecx}(\vecx_0)$. The generative process in BayLIME and BayesLIME/BayesSHAP corresponds to the Bayesian version of the weighted least squares formulation of LIME and KernelSHAP given in \eqref{eq:basicLIME}. In particular, BayesLIME/BayesSHAP utilizes the uncertainty associated with the feature importance and an error term for driving focused sampling.

\section{Active learning for XAI}

In the context of post-hoc explainable models, we consider an active learning strategy where an active learner has three components: $(f_e, \alpha, \mathcal{X})$. Here, $f_e$ is the explainer model initially trained on the index sample $(\vecx_0,f_p(\vecx_0)) \in \mathcal{D}$, $\alpha(\cdot)$ is the querying function that, given a current labeled set $\mathcal{D}$, decides which instance to query next. The post-hoc XAI based active learner returns an explainer model $f_e$ after a fixed number of queries.

\subsection{Gaussian Process (GP) Surrogate Model}
\label{sec:GPDetails}
We employ a Gaussian process-based surrogate model for locally emulating the black-box prediction model. GP has vital advantages over commonly used linear methods in post-hoc XAI, including the ability to fit highly nonlinear functions with minimal risk of over-fitting and capability for uncertainty estimation and quantification \cite{GPRasmussen}. In particular, it \cite{NealGP} shows that the neural networks with infinite hidden units converge to a GP. Gaussian process regression uses a prior distribution on $f_e(\cdot)$ that captures our beliefs about the behavior of $f_p(\cdot)$, and updates this prior with sequentially acquired data $\vecx$ using an active learning-based acquisition function $\alpha(\cdot)$. Surrogate data leads to a likelihood function $p(\mathcal{D}|f_e)$, which is combined with the prior to obtain the posterior distribution, which in turn is the new prior distribution incorporating both our prior beliefs and information from the surrogate data points. 

We model the prior probability distribution of $f_{e}$ using a Gaussian process prior with zero mean and covariance matrix $\mathbf{K}_0$, i.e., $p(f_{e}) = \textnormal{GP}(f_{e}; \mathbf{0}, \mathbf{K}_0)$, which is a favored process prior due to the closed-form nature of posterior density. Conventionally, for $n$ samples, the $(p,q)$-the entry of a covariance matrix is given by  $k(\vecx_p,\vecx_q)$, where $k(\cdot,\cdot)$ is the covariance function. Let the $n$-th ($n < N$) target be $y_n = f_p(\vecx_n)$, and, $\vecy = [y_1, y_2,\hdots , y_n]$. The predictive distribution of $y_{n+1}$ conditioned on $\vecx_{n+1}$ is computed as:
\begin{equation}
    p(y_{n+1}|\vecx_{n+1}, \mathcal{D}_n) = \mathcal{N}(\mu_n(\vecx_{n+1}), \sigma_n^2(\vecx_{n+1})),
\end{equation}
where $\mathcal{D}_n \subset \mathcal{D}$ ,posterior mean is given as $\mu_n(\vecx) = k_n^T(\vecx)\matK_0^{-1}\vecy_n$ and the posterior  variance $\sigma_n^2(\cdot)$ is given by  $\sigma_n^2(\vecx) = k(\vecx,\vecx) - k_n^T(\vecx)\matK_0^{-1} k_n^T(\vecx)$, where  $k_n(\vecx) =  [k(\vecx, \vecx_1), \hdots , k(\vecx, \vecx_n)]$. The choice of the covariance function $k(\cdot,\cdot)$ is crucial, especially in the context of explainable AI as it encodes the inductive bias of the GP. In particular, covariance functions must be able to capture the information about the structure of the underlying black-box function. The proposed framework is capable of incorporating any covariance function.  We choose the well-known Matern kernels and a linear kernel for illustrating the results. 

\begin{figure*}
    \centering
    \begin{subfigure}[t]{\linewidth}
    \centering    \includegraphics[width=1\textwidth]{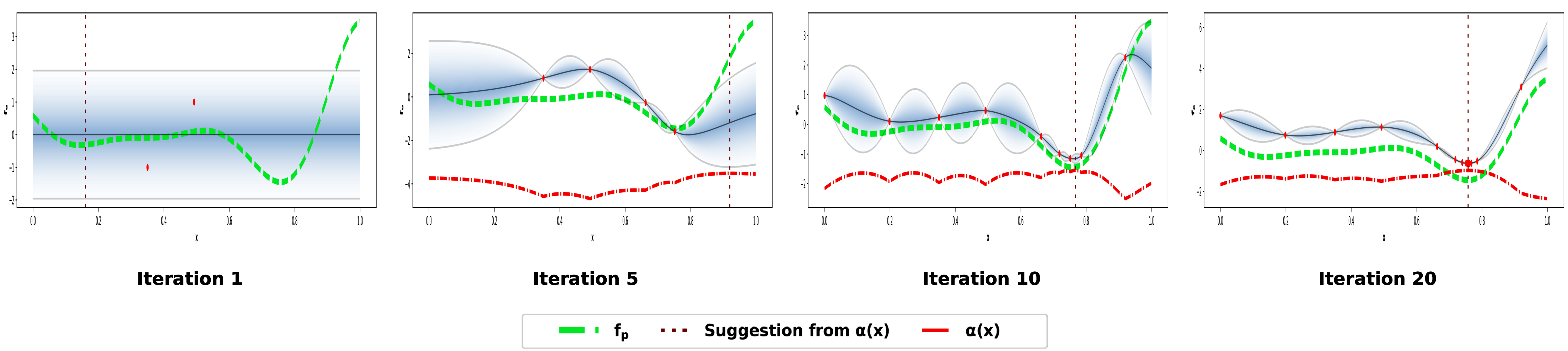}
        \caption{UCB(Upper confidence bound): $\argmax_{\vecx} \alpha(\vecx|\mathcal{D}_{n}) = \argmax_{\vecx} {\vecx + k\sigma_n(\vecx)}$}
  \end{subfigure}
    \begin{subfigure}[t]{\linewidth}
    \centering    \includegraphics[width=1\textwidth]{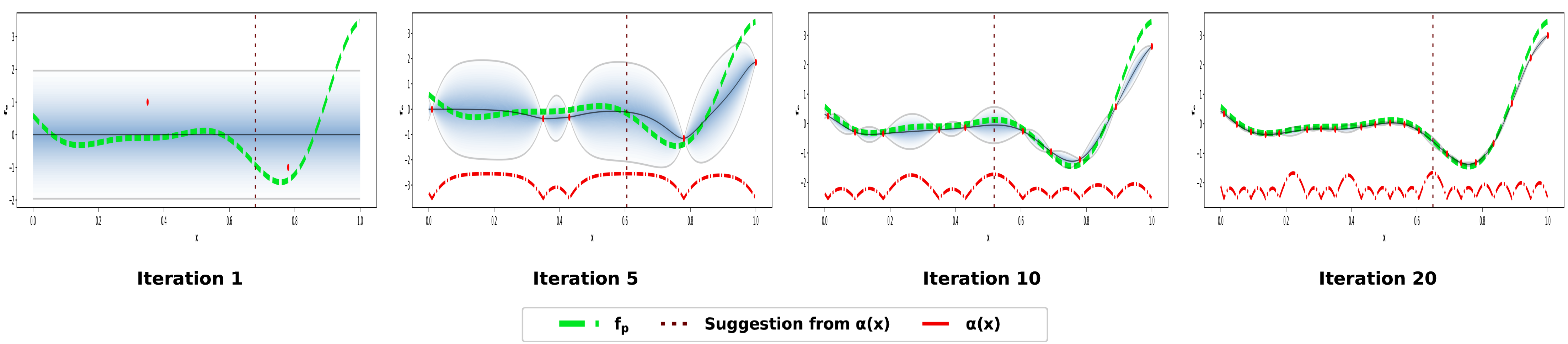}
    \caption{Uncertainty Reduction: $\argmax_{\vecx} \alpha(\vecx|\mathcal{D}_{n}) = 
    \argmax_{\vecx} \sigma_n(\vecx)$}
  \end{subfigure}
  \begin{subfigure}[t]{\linewidth}
    \centering    \includegraphics[width=1\textwidth]{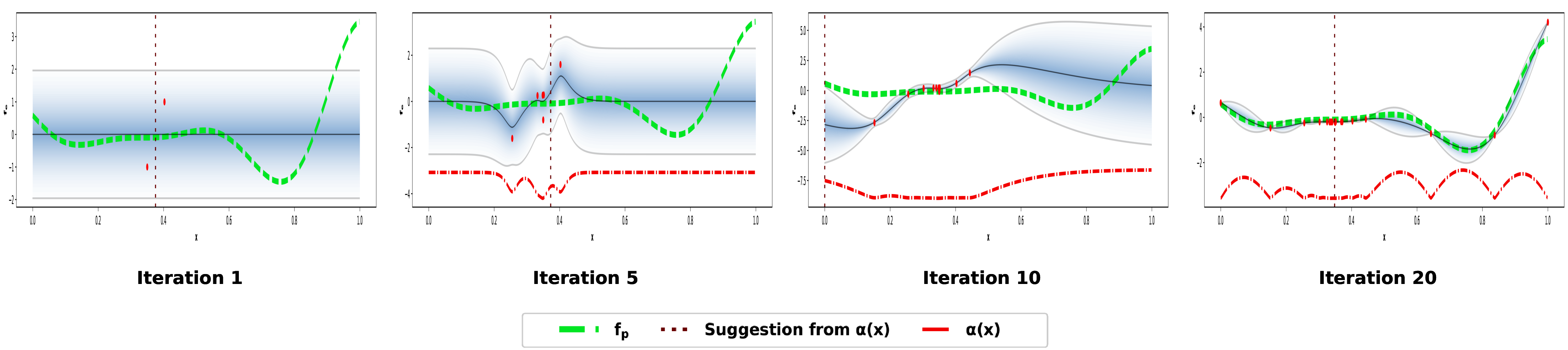}
    \caption{Faithful Uncertainty Reduction: $\argmax_{\vecx} \alpha(\vecx|\mathcal{D}_{n}) = 
    \argmax_{\vecx} -\left\Vert\left(\vecx - \vecx_0 - \frac{\overline{\sigma}\epsilon}{\log(n)}\rm\right) \right\Vert_2 + \sigma_n(\vecx)$}
  \end{subfigure}
    \caption{Optimizing the one-dimensional Forrester function\cite{forrester}: Comparison of different iterations of UCB, UR and the proposed FUR acquisition functions}
    \label{fig:alpha_iteration}
    
\end{figure*}
\subsection{Acquisition Functions}

The big question in post-hoc explainable AI is, \emph{Given the instance to be explained, how do we choose the surrogate data samples?}. Active learning acquisition functions provide a principled answer to this question by determining the sampling regions while balancing exploration and exploitation trade-offs. Maximizing an acquisition function $\alpha(\vecx|\mathcal{D}_n)$ yields the selection of the next point at which to evaluate the objective function, i.e.,
\begin{equation}
    \vecx_{n} = \argmax_{\vecx} \alpha(\vecx|\mathcal{D}_{n-1}).
\end{equation}

Typically, the acquisition function is designed to be an  inexpensive function that can be easily evaluated at a given point. In the context of GP regression, the prior is Gaussian, and the pairwise kernel function specifies the similarity measure between an unobserved point compared to each of the observed points. The posterior is obtained in closed form using the prior likelihood and the kernel function. The mean  $\mu(\vecx)$ and the standard deviation $\sigma(\vecx)$ function forms the basis for designing popular acquisition functions \citep{GPRasmussen}. Along these lines, the most popular acquisition functions include \citep{BORef1} expected improvement (EI), lower/upper confidence bound (LCB/UCB), uncertainty reduction (UR), etc.

Consider the UCB  and UR acquisition functions that lead to a new observation being selected as
\begin{align}
     \textnormal{UCB}&: \vecx_{n} = \argmax_{\vecx} \mu_{n-1}(\vecx) + \sqrt{\beta_{n}}\sigma_{n-1}(\vecx)\nonumber\\
     \textnormal{UR}&:\vecx_{n} = \argmax_{\vecx} \sigma_{n-1}(\vecx).
     \label{eq:UCB-UR}
\end{align}
Note that the goal of these acquisition functions is to obtain the best fit GPR based on the points in a pre-defined domain. 

An essential criterion in locally explainable models is local fidelity. Since it is often impossible for a post-hoc  explanation to be globally faithful unless the complete description of the black-box model and the dataset is available, it is expected that post-hoc models are at least locally faithful, i.e., it must modulate the model behavior in the vicinity of the index sample. Note that the acquisition functions in \eqref{eq:UCB-UR} are not directly applicable in the context of locally explainable models since they do not incorporate anything towards local fidelity. To make this intuition concrete, consider the optimization of the one-dimensional Forrester function ($f(x) = (6x-2) \sin(12x-4)$ in the domain $[0,1]$) in Fig. \ref{fig:alpha_iteration}. It can be seen that the points chosen over several iterations do not necessarily lie in the locality of the index sample $\vecx_0$. In the case of UCB, the acquisition function explores initially but soon turns exploitative and reaches the global optima. However, UR is more exploratory and decreases uncertainty in the entire interval. Hence, it is evident that for any acquisition function to obtain the best possible points in the locality of the index sample, this function must specify the locality. 

In the following Section, we propose a novel acquisition function that fulfills the criterion of local fidelity while providing the advantages that entail GP. 

\section{Proposed Acquisition Function: UnRAvEL}

We propose a novel acquisition function, which we refer to as \emph{Faithful Uncertainty reduction} given as:
\begin{align}
    \vecx_n = \argmax_{\vecx} \underbrace{-\left\Vert\left(\vecx - \vecx_0 - \frac{\overline{\sigma}\epsilon}{\log(n)}\rm\right) \right\Vert_2}_{\textnormal{T1}} + \underbrace{\sigma_n(\vecx)}_{\textnormal{T2}},
    \label{eq:FUR}
\end{align}
where $\overline{\sigma}$ is the empirical mean of the standard deviation of individual features in training data, $\epsilon \sim \mathcal{N}(0, 1)$, $\vecx_0$ is the index sample, and $\sigma_n(\vecx_n)$ is the standard deviation of $f_e$  obtained until the $n$-th sample $\vecx_n$. The novel acquisition function has the following properties:
\begin{itemize}
    \item \textbf{Local Fidelity-Exploration Trade off}: From \eqref{eq:FUR}, we see that T1 ensures that the sampled point $\vecx$ is the nearest in the vicinity of the index sample (a hyper-sphere) whose radius is stochastic but  decays with $n$. The term $\frac{\overline{\sigma}\epsilon}{\log(n)}$ gives wiggle room for the choice of $\vecx$ and allows FUR to explore in the vicinity of $\vecx_0$. By incorporating $\overline{\sigma}$, we ensure that exploration around $\vecx_0$ does not diverge. On the other hand, T2 ensures that the sampled point maximizes information gain \cite{UCBsrinivas2010gaussian}. Hence FUR ensures that local fidelity is achieved in the region close to $\vecx_0$ while ensuring that the acquisition function explores in this vicinity. This nature of FUR, as compared to UCB and UR can be observed in Fig.~\ref{fig:alpha_iteration}. It can be seen for the one-dimensional example that FUR samples in the vicinity of the index-sample.
    \item \textbf{Sample Efficiency}: In T1, the term $\frac{\overline{\sigma}\epsilon}{\log(n)}$ ensures that as $n$ increases, $\vecx$ goes closer to the index sample. Furthermore, as $n$ increases, the information gain from $\sigma_n(\vecx)$ (T2) decreases and hence, T1 dominates. Consequently, we observe that samples co-incide with $\vecx_0$ and FUR converges to an explainable ML model defined in the locality of $\vecx_0$. For the one-dimensional case, this is illustrated in Fig.~\ref{fig:alpha_iteration} at iteration $20$.
    \item \textbf{Surrogate dataset generator}: The proposed acquisition function is used in conjunction with a GP for  obtaining the posterior variance. Hence, UnRAvEL can be viewed as a method for generating the surrogate dataset, which can generate informative samples around the index sample. The surrogate dataset generated with an appropriate kernel can be used with a sparse linear model for obtaining importance scores. 
    \item \textbf{Hyperparameters}: To compensate for the low capacity linear model, LIME and SHAP employ various hyperparameters that are difficult to estimate. The impact of SHAP hyperparameters like K-means center and LIME hyperparameters like kernel-width for weighing local samples have been studied thoroughly in the literature\cite{OptiLIME, XXAI, BayLIME}. To counter this, methods like BayLIME\cite{BayLIME} employ hyperpriors on kernel width, whereas methods like OptiLIME\cite{OptiLIME} introduce a novel hyperparameter-based optimization function. However, none of these techniques are entirely hyperparameter free. In FUR, the structure of the acquisition function allows an interpretable and sample-efficient search process, shaped only by the kernel function where the samples are information-rich and in the vicinity of the sample being explained.
\end{itemize}

\subsection{Importance Scores using UnRAvEL}
GP models are inherently non-parametric, and hence, to measure importance estimates, we rely on the kernel parameters and the posterior distribution. The prominent methods that can be used to obtain importance scores are as follows:\\
    \textbf{ARD Explainer} \citep{ARDpaananen}: This explainer is obtained by using ARD-based kernel function \cite{ARDpaananen}, where the ARD kernel contains an individual length-scale parameter for each input covariate, and the relevance of each covariate is determined by the estimated length-scale value, large (small) values indicate lower (higher) relevance. \\
    \textbf{UnRAvEL-LIME}: In case UnRAvEL is used for surrogate dataset generation, one can use a sparse linear model for explanations. We refer to this as UnRAvEL-LIME as this leads to LIME-like importance scores using a more minor but informative surrogate dataset.
The pseudo-code for UnRAvEL is as given in Algorithm~1.

\begin{algorithm}
 \caption{UnRAvEL: Uncertainty driven Robust Active Learning Based Locally Faithful Explanations}.
 \begin{algorithmic}[1]
 \REQUIRE Black-box model $f_{p}$, Instance $\vecx_0 \in \mathbb{R}^{d}$, $\bar{\sigma}$, $\bm{\sigma_{\mathcal{D}}}$ = $[\sigma_1, \hdots, \sigma_n]$, Maximum iterations $L$, Acquisition function $\alpha(\cdot)$
 \STATE Initialize $\mathcal{D}$ using $(\vecx_0, f_{p}(\vecx_0))$
 \STATE Set exploration domain for $\alpha(\vecx)$: $\vecx \in [\mathbf{x}- \bm{\sigma_{\mathcal{D}}}, \mathbf{x} + \bm{\sigma_{\mathcal{D}}}]$.\\ \STATE Initialize the GPR and the ARD kernel. \\
  \textbf{Active Learning Routine:}
  \FOR {$l = 1$ to $L$}
  \STATE Obtain $\vecx_{l+1}$ by optimizing $\alpha(\vecx)$. $\mathcal{D} \leftarrow \mathcal{D} \cup (\vecx_{l+1},f_p(\vecx_{l+1})$. 
  \STATE Train the GPR based $f_{e}$ model on $\mathcal{D}$.
  \ENDFOR
  \RETURN Surrogate data $\mathcal{D}$, Importance scores using ARD-explainer or UnRAvEL-LIME.
 \end{algorithmic} 
 \end{algorithm}
 
\subsection{Local Optimizer Viewpoint}
It is common to perform regret-based convergence analysis for GPR in scenarios where the final goal is to find a global optimum of optimization problems \cite{srinivasLCB}. These convergence guarantees are only valid when the global optimizer of the acquisition function is obtained at each round and selected as the following query point. In the context of post-hoc locally explainable models, a local optimizer of the acquisition function is used since searching for the global optimum is not the end goal. Straightforward regret analysis is challenging, and hence, performance analysis based on the behavior of the proposed local optimizer in terms of instantaneous regrets as compared to a global optimizer is explored similarly to \cite{kim2020local}. 

We denote by $\vecx_{n,g}$ as the optimizer in round $n$, as determined by a global acquisition method $\vecx_{t,g} = \argmax_{\vecx}\alpha_g(\vecx|\mathcal{D}_{n-1})$, where the reference global acquisition function can be $UCB$ or $UR$ whose convergence is well-known \cite{srinivasLCB}. Further, we denote $\vecx_{n,l}$ as the optimizer (local) of the acquisition function $\alpha_l(\vecx|\mathcal{D}_{n-1})$ in round $n$. Suppose that $\vecx^*$ is the true global minimum of the objective function, and $\vecx_n$ be a maximum of acquisition function in round $n$, determined by either a global or local optimization method. The regret for round $n$ is defined as $r_n = f_e(\vecx_n) - f_e(\vecx^*)$. In this context, we define instantaneous regret differences for an local optimizer $\vecx_{n,l}$ and global optimizer $\vecx_{n,g}$ as $|r_{n,g} - r_{n,l}|=|f_e(\vecx_{n,g})-f_e(\vecx_{n,l})|$. If $f_e$ is Lipshitz continuous and if the search-domain $\mathcal{X}$ is compact, the following theorem holds.

\begin{theorem}
Given $\zeta \in [0, 1)$, $\epsilon_l, \delta_n > 0$, the probability of the regret difference $|r_{n,g} - r_{n,l}|$ for a local optimizer $\vecx_{n,l}$ in round $t$ fulfils the following: 
\begin{equation}
    \Pr(|r_{n,g} - r_{n,l}|<\epsilon_l) \geq 1 - \zeta,
\end{equation}
where $\zeta = 1- \frac{d^n_{0,g}}{\eta_1} - \frac{(1-\beta_0)\gamma}{\eta_1} - \frac{M}{\eta_2}$, $\gamma = \max_{xi,xj \in \mathcal{X}} ||x_i- x_j||_2$, $\eta_1\eta_2 = \epsilon_l+\delta_n$, $\delta_n = |f_e(\vecx_{n,g}) - f_e(\vecx_0)|$, $d^n_{0,g} = \Ex||\vecx_0 - \vecx_{n,g}||_2$ and $\beta_0$ is the probability that the local optimizer co-incides with the index sample.
\end{theorem}
The proof is given in the appendix. Theorem~1 suggests that the regret of a global optimizer defined on $\mathcal{X}$ is comparable to any FUR-like local optimizer as a function of parameters such as $\gamma$, $\beta_0$, $M$ and $d^n_{0,g}$. If the search domain $\mathcal{X}$ is relatively small, then the local and global regret may be close. Further, as $\beta_0$ is close to one as $n$ increases. However, $d^n_{0,g}$ is irrespective of the local optimizer and solely depends upon the global optimization technique chosen. 

\section{Experimental Results}
In this Section, we demonstrate the efficacy of the proposed UnRAvEL framework on publicly-available real-world  datasets concerning attributes like sample efficiency and stability (consistency) in repeated explanations. We employ tabular and image datasets and consider different black-box models for an explanation. To validate the surrogate data generated by UnRAvEL, we motivate the usage of such data in existing explanation pipelines such as \cite{LIME}.

\begin{table}
        \caption{Description of datasets. 'R' denotes regression and 'C' denotes classification task}
  \begin{tabular}{cccccc}
    \toprule
    \textbf{Dataset} & \textbf{Task} &$p$ & {$n_{train}$} & $n_{total}$ & $R^2$ score \\
    \midrule
    \textbf{Parkinson's} & C & 22 & 195 & 175 & 0.80\\
    \textbf{Cancer} & C & 30 & 512 & 569 & 0.98\\
    \textbf{Adult} & C & 14 & 30162 & 45222 &  0.84\\
    \textbf{Bodyfat} & R & 14 & 226 & 252 & 0.99\\
    \textbf{Boston} & R & 13 & 455 & 506 &  0.92\\
    \hline
    \end{tabular}

    \label{tab:dataset_description}
\end{table}
\begin{figure*}
    \centering
    
    \begin{subfigure}[t]{0.22\linewidth}
    \includegraphics[width=0.85\textwidth, height=0.15\textheight]{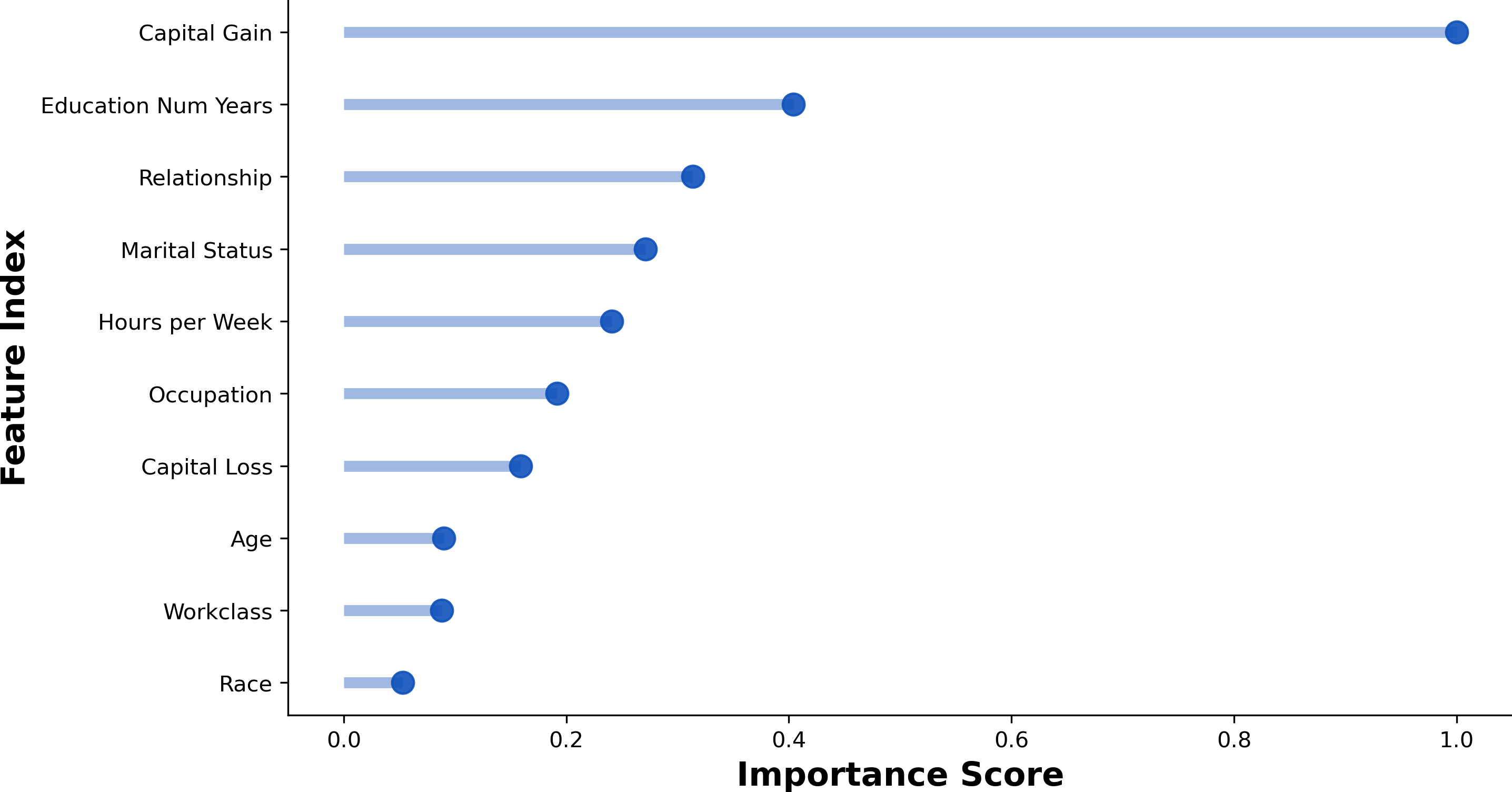}        \caption{Explaining $f_p(\vecx)=0.91$}
    \end{subfigure}
  \begin{subfigure}[t]{0.22\linewidth}
    \centering    \includegraphics[width=0.85\textwidth, height=0.15\textheight]{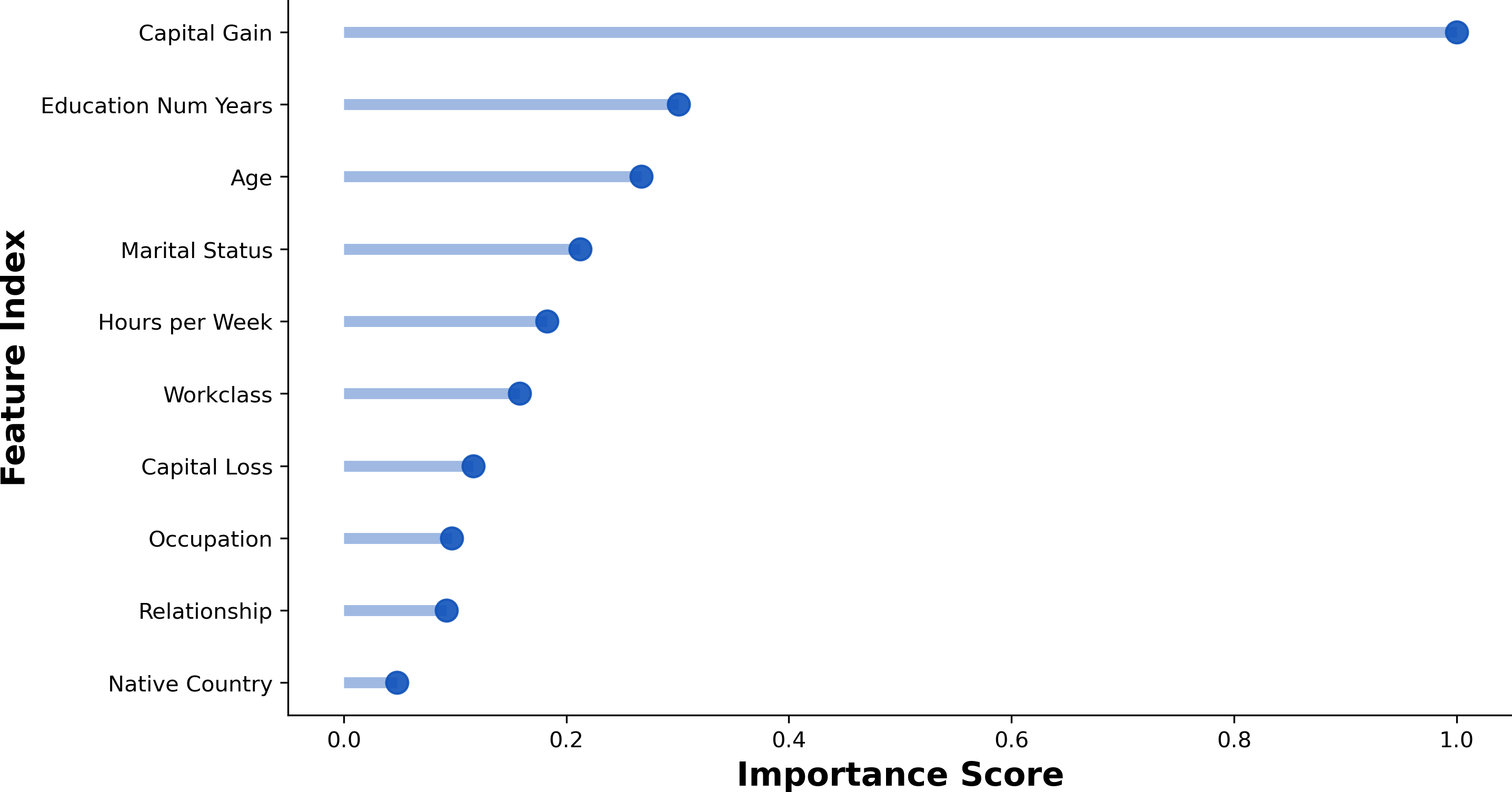}
    \caption{Explaining $f_p(\vecx)=0.05$}
  \end{subfigure}
  \begin{subfigure}[t]{0.22\linewidth}
    \centering    \includegraphics[width=0.85\textwidth, height=0.15\textheight]{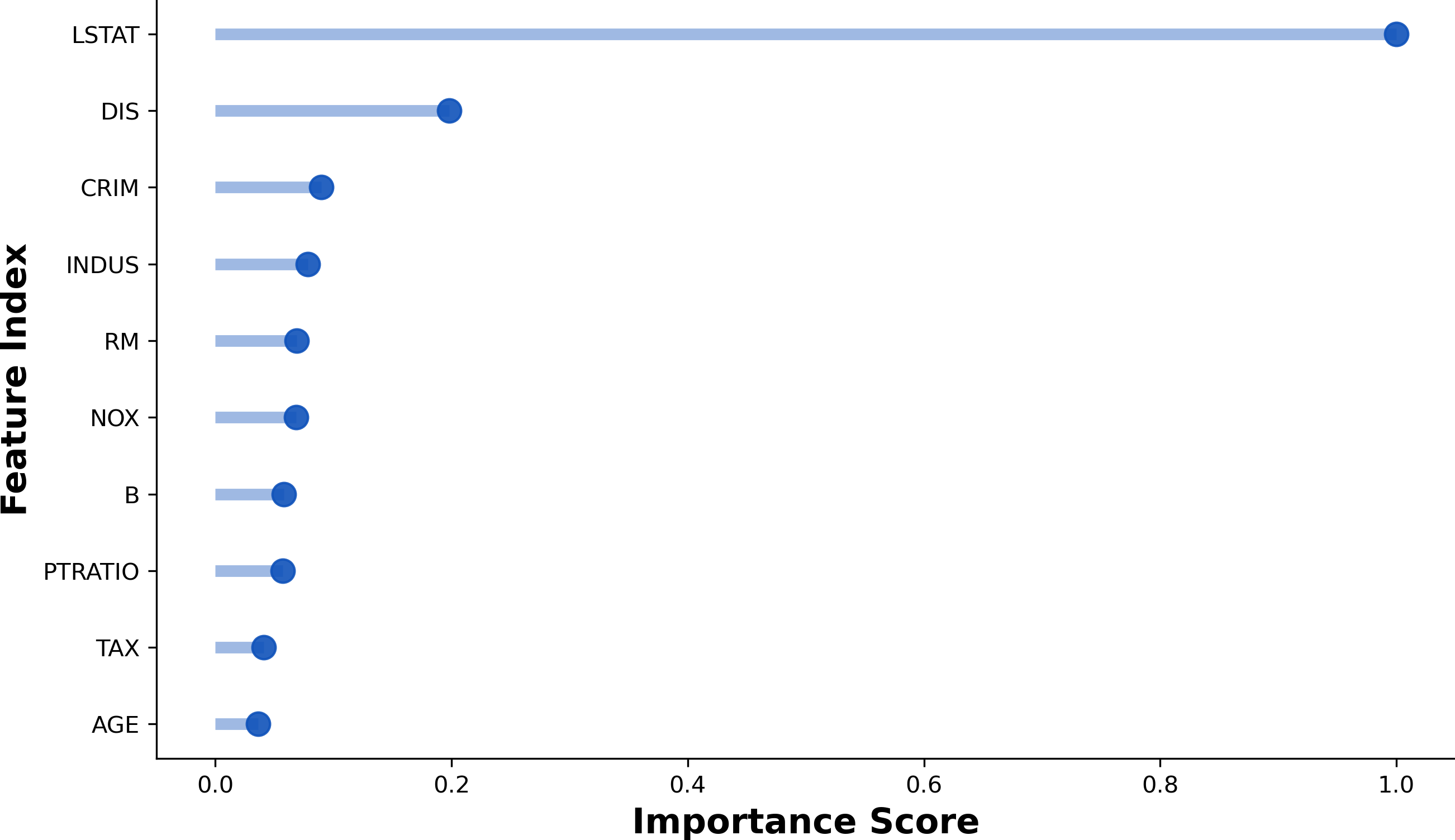}
    \caption{Explaining $f_p(\vecx)=35.73$}
  \end{subfigure}
  \begin{subfigure}[t]{0.22\linewidth}
    \centering    \includegraphics[width=0.85\textwidth, height=0.15\textheight]{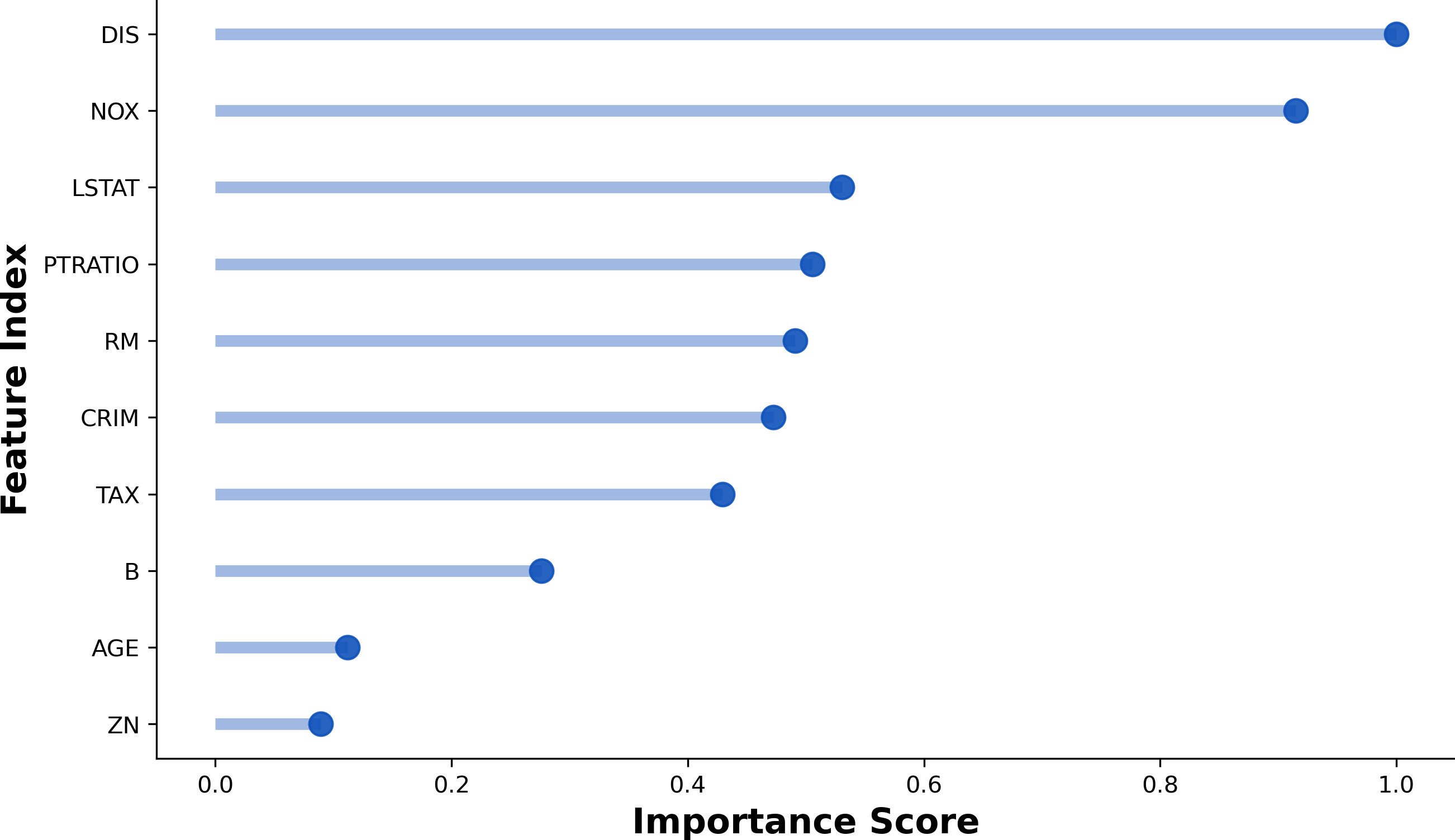}
    \caption{Explaining $f_p(\vecx)=11.85$}
  \end{subfigure}
  \caption{Mean Importance scores for two samples from the Adult dataset(a-b) and Boston house pricing dataset(c-d). Both the datasets have sensitive variables in their setup.}
    \label{fig:mip_comparison}
\end{figure*}

\subsection{Datasets and Pre-processing}

We chose five distinct datasets from the UCI Machine Learning repository \cite{UCI} for the tabular data-based experiments based on their usage in the relevant literature \cite{BayLIME,DLIME,ALIME}, feature novelty, and prediction task. The employed pre-processing pipeline was the same for all the datasets, involving removing all missing values, frequency encoding of categorical features, followed by a standardization procedure per feature. The description of the datasets are as follows: 
\begin{itemize}
    \item Parkinson's: The Parkinson's classification dataset consists of $195$ biomedical recordings of patients suffering and free from Parkinson's disease. With $22$ unique features per recording, the task is to classify whether a given patient has Parkinson's or not.
    
    \item Cancer: The Breast cancer classification dataset consists of $569$ entries, with $30$ features computed from an image of a breast mass, describing characteristics of the cell nuclei present in the image to predict if the cancer is malignant or not. 
    
    \item Adult: The Adult Income classification dataset consists of $14$ features describing the educational, financial, and racial recordings of $45222$ adult individuals. The task is to predict whether the income of a given individual is above $50,000$ US Dollars or not.
    
    \item Boston: The Boston house pricing regression dataset consists of $506$ entries with $13$ features of homes from various suburbs located in Boston. The task is to predict the price of the house based on neighborhood features.
    
    \item Bodyfat: The Bodyfat regression dataset consists of $14$ features, each depicting various physical properties like determining the body fat percentage of $252$ men.
    
\end{itemize}
For simulating the prediction models, we used a Support Vector Classifier for all the classification datasets and an Extra Trees Regressor for all the regression datasets based on their accuracy performance on the test set. A summary of the dataset and prediction model statistics can be found in Table \ref{tab:dataset_description}.

\begin{figure}
    \centering
    \begin{subfigure}[t]{0.45\linewidth}
    \includegraphics[width=1.05\textwidth]{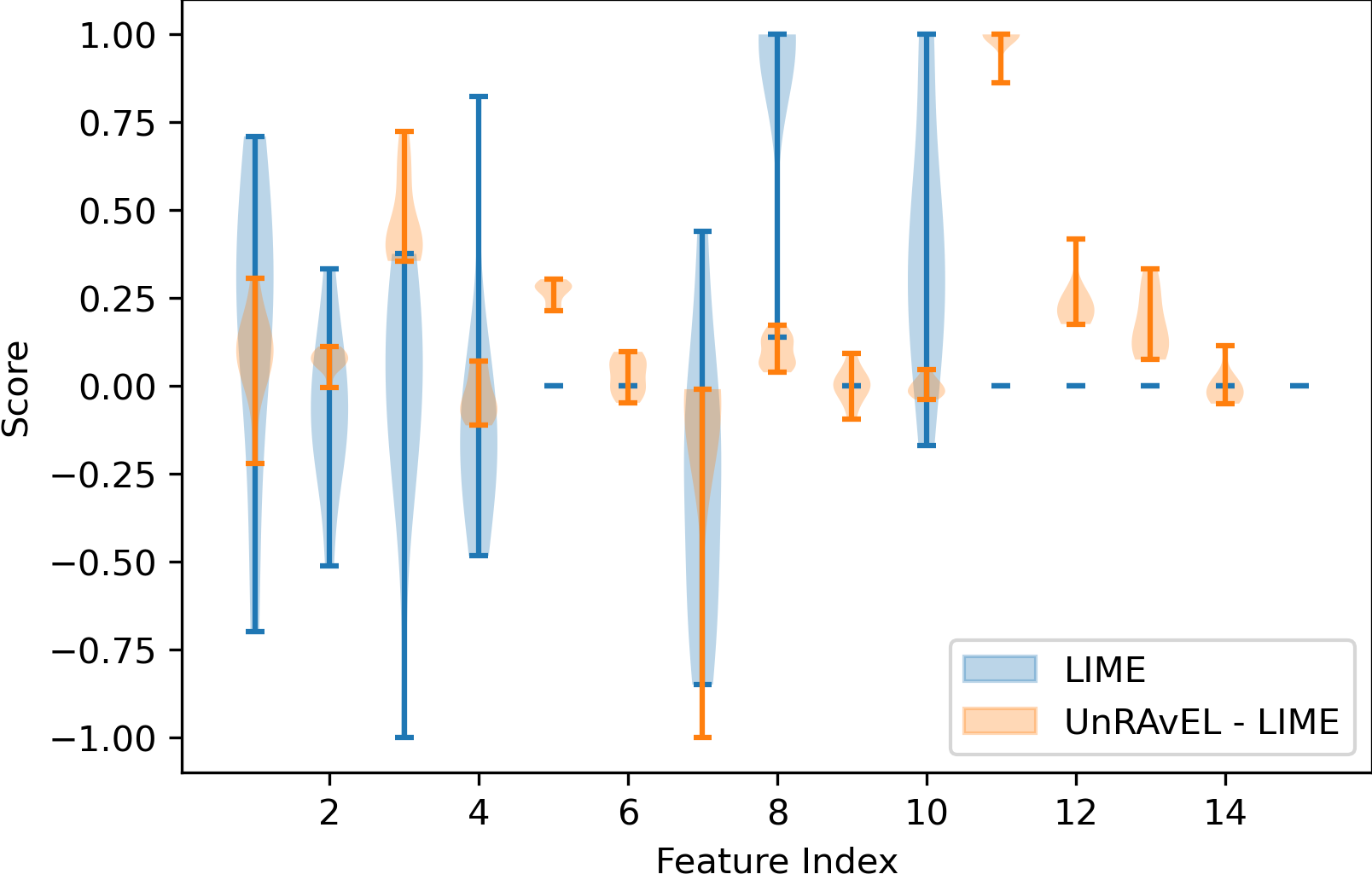}
        \caption{$f_p(\vecx)=0.52$}
    \end{subfigure}
  \begin{subfigure}[t]{0.45\linewidth}
    \centering    \includegraphics[width=1.05\textwidth]{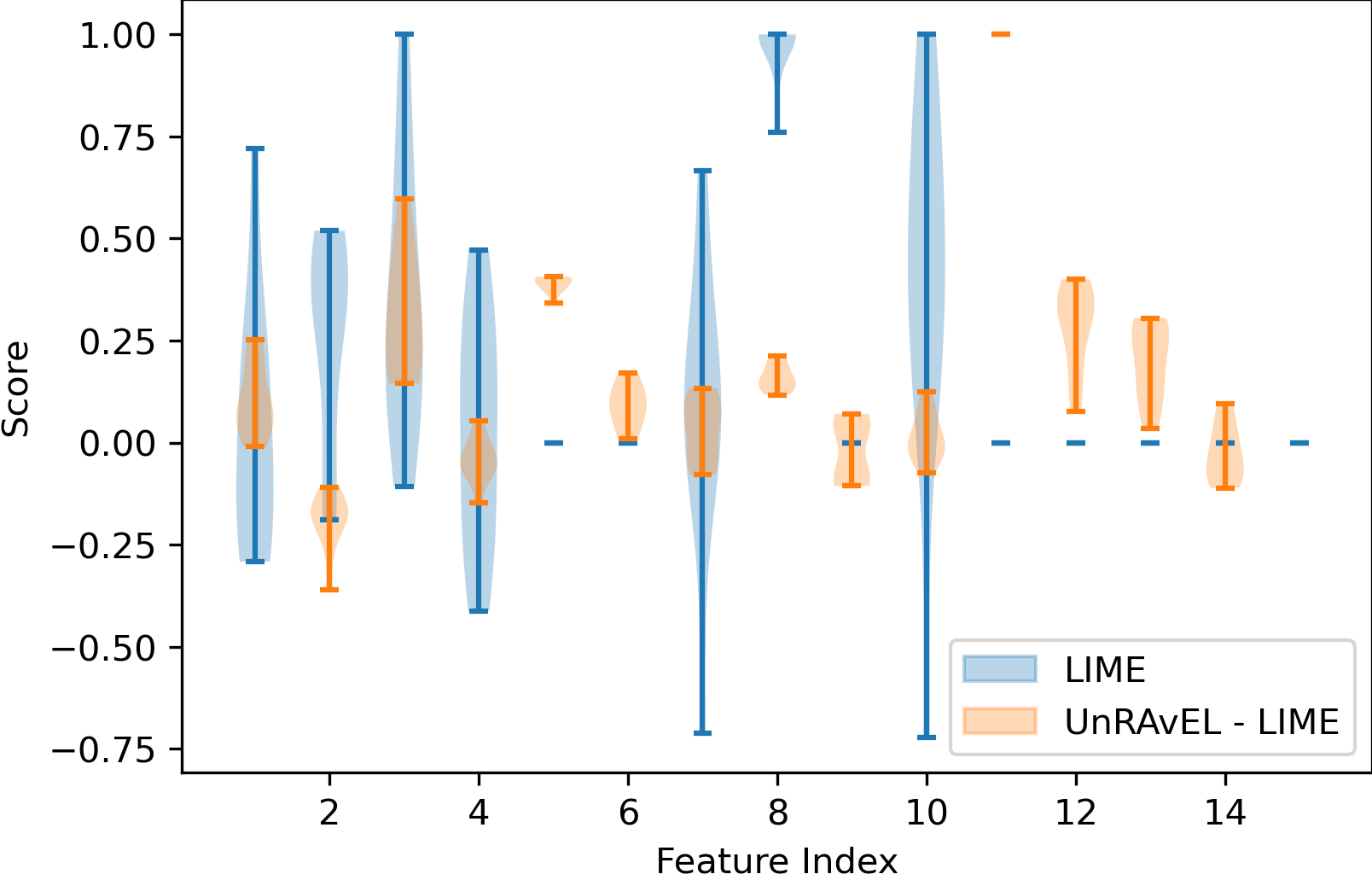}
    \caption{$f_p(\vecx)=0.78$}
  \end{subfigure}
    \caption{Comparison of the explanations generated by LIME and UnRAvEL-LIME, As can be seen, the explanations generated by UnRAvEL-LIME are less uncertain than their LIME counterpart}.
    \label{fig:surrogate_data_comparison}
\end{figure}

\begin{figure*}
\centering
\begin{tabular}{cc@{\hskip 1cm}c@{\hskip 1cm}cc}
\setlength\tabcolsep{0pt}
 \textbf{Target Image} &
 \textbf{Grad-CAM Output} & 
 \textbf{UnRAvEL(100 samples)} &
 \textbf{LIME(100 samples)} & 
 \textbf{LIME(10,000 samples)}\\
 \includegraphics[width=0.15\textwidth]{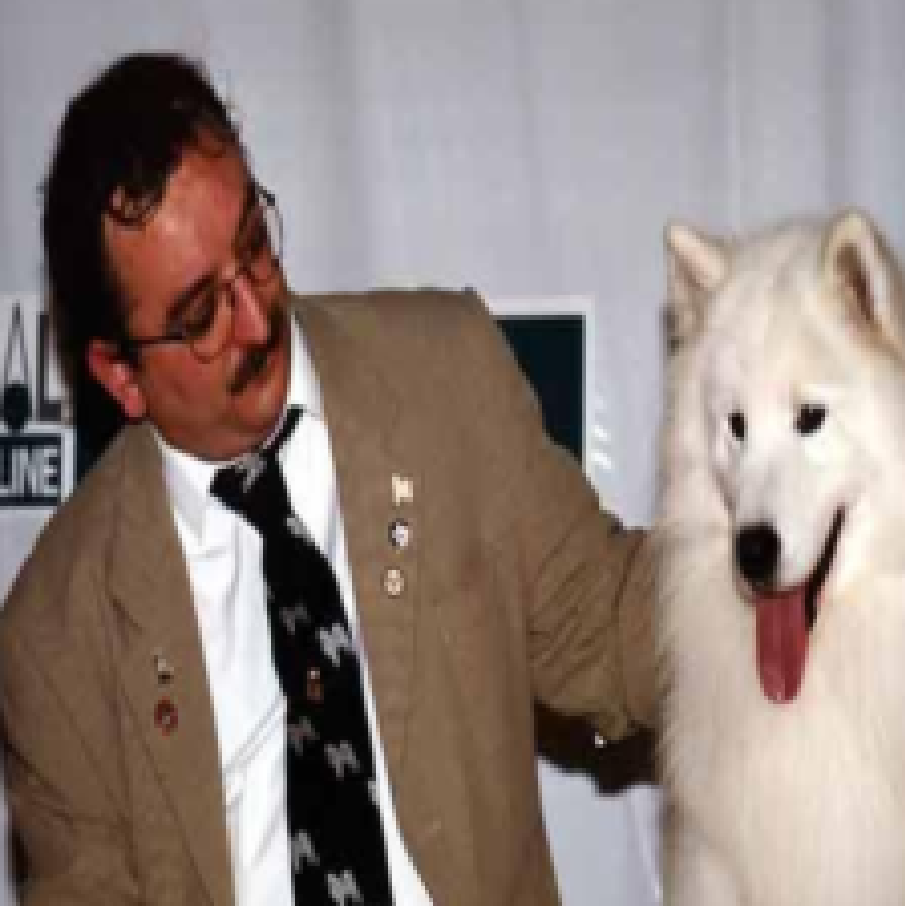} &
 \includegraphics[width=0.15\textwidth]{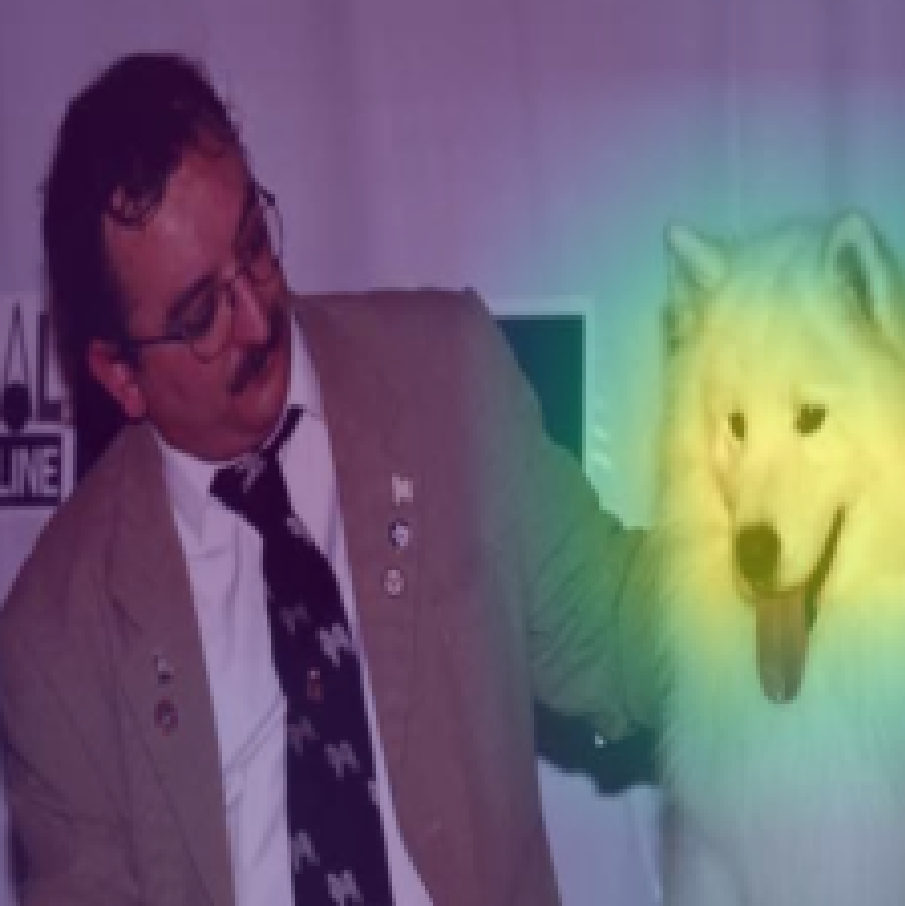} &
\includegraphics[width=0.15\textwidth]{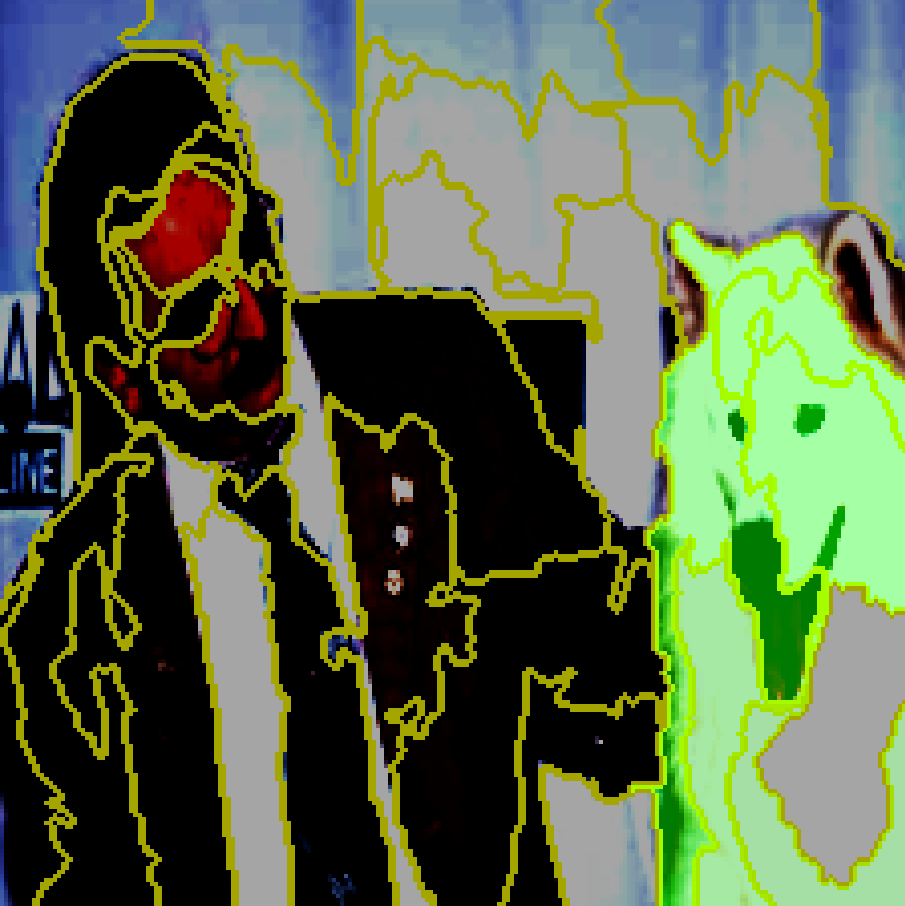} &
\includegraphics[width=0.15\textwidth]{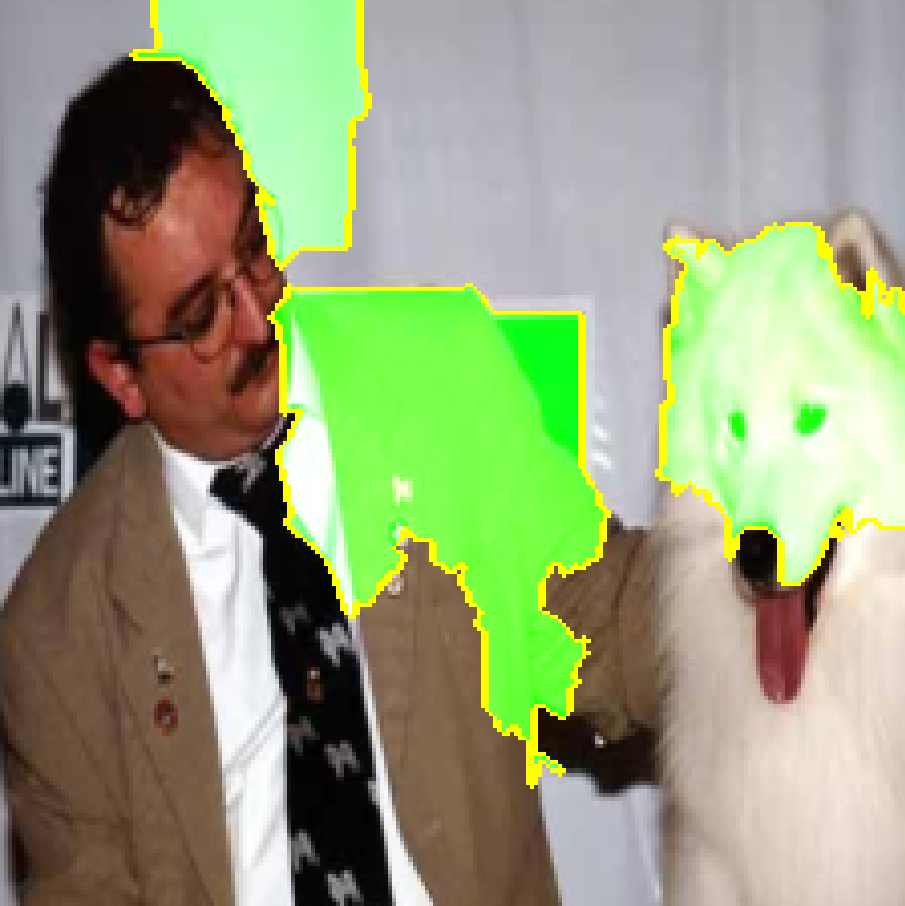} &
\includegraphics[width=0.15\textwidth]{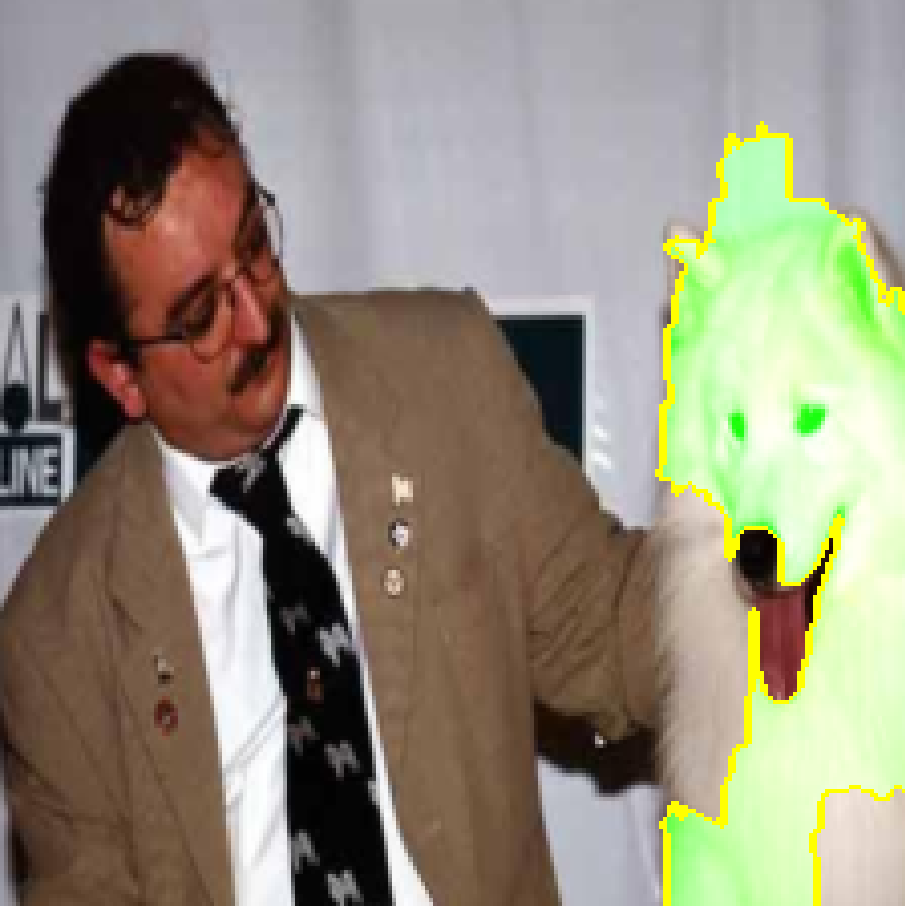}\\
\footnotesize P('Samoyed') = 0.77&&&&\\
 \includegraphics[width=0.15\textwidth]{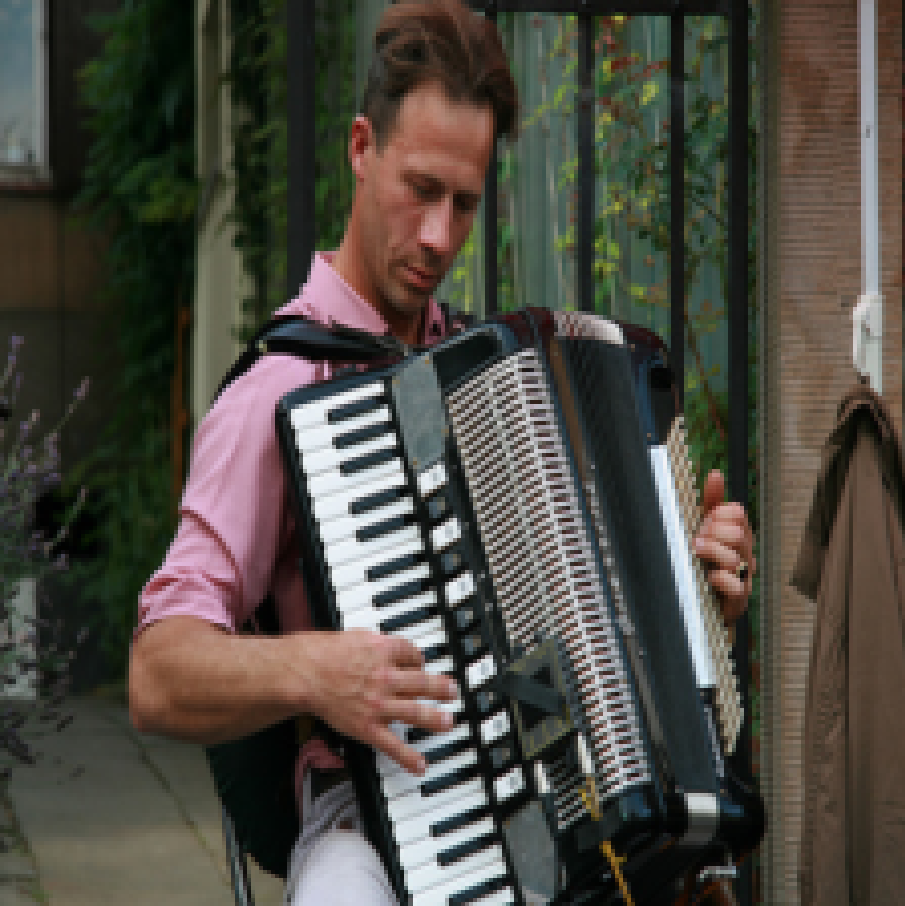} &
\includegraphics[width=0.15\textwidth]{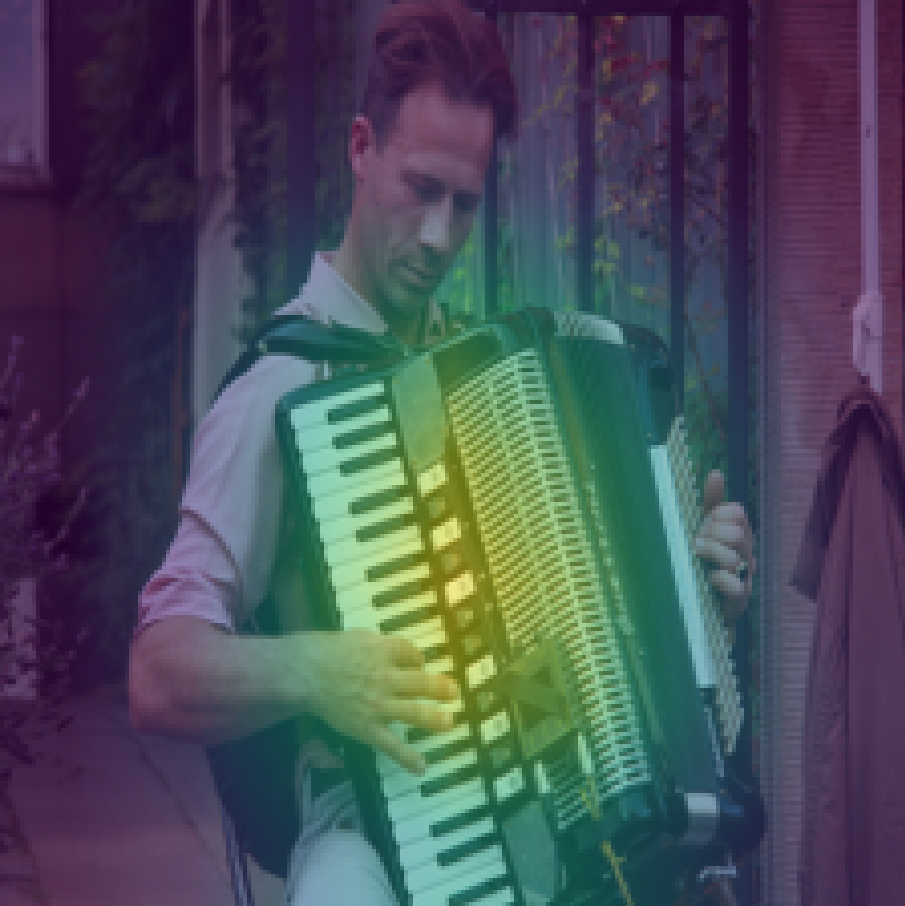} & \includegraphics[width=0.15\textwidth]{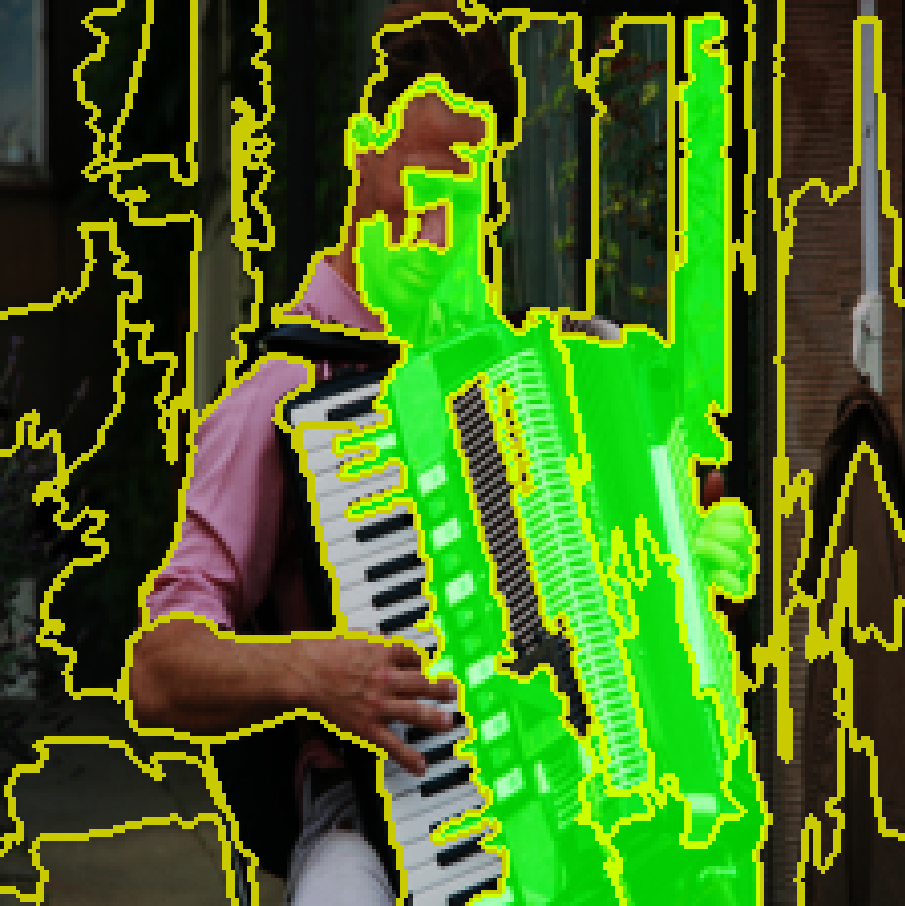} &
\includegraphics[width=0.15\textwidth]{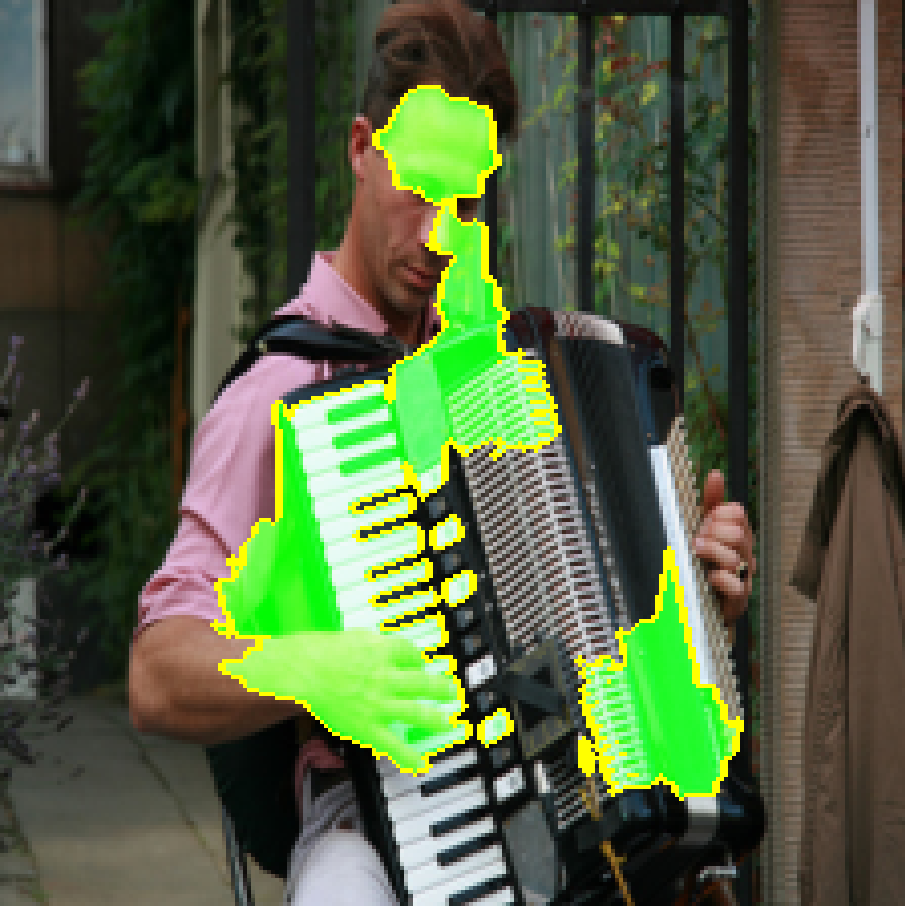} & \includegraphics[width=0.15\textwidth]{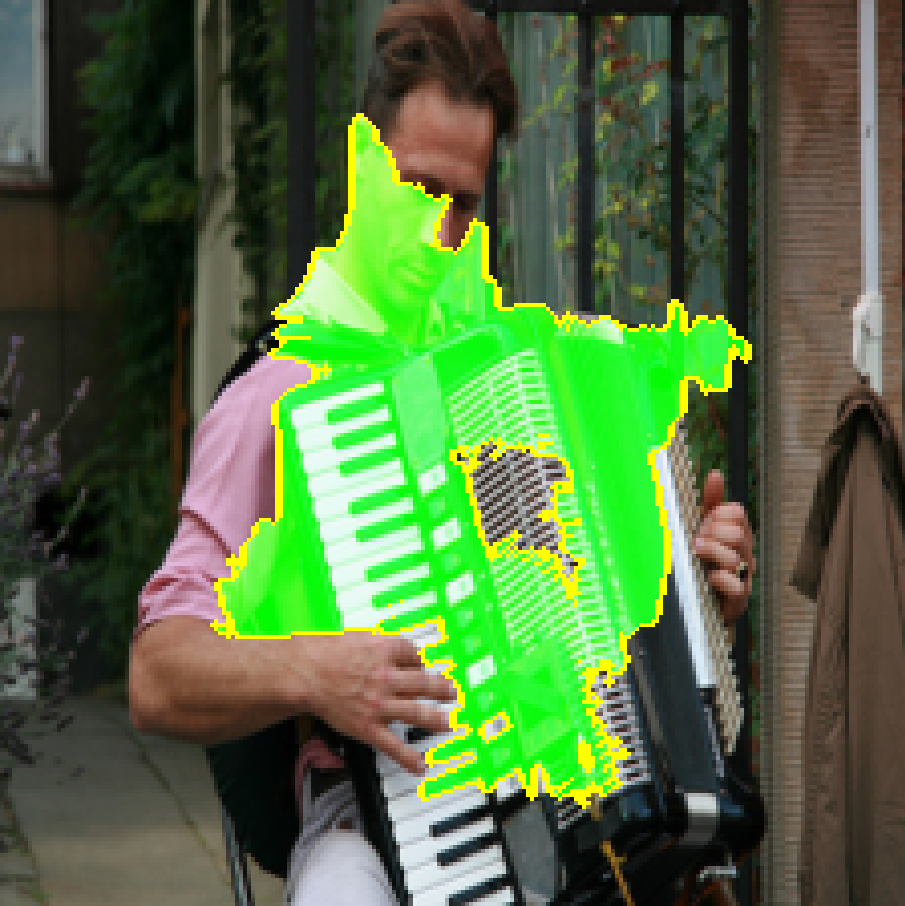}\\
\footnotesize P('Accordion') = 0.99&&&&\\
 \includegraphics[width=0.15\textwidth]{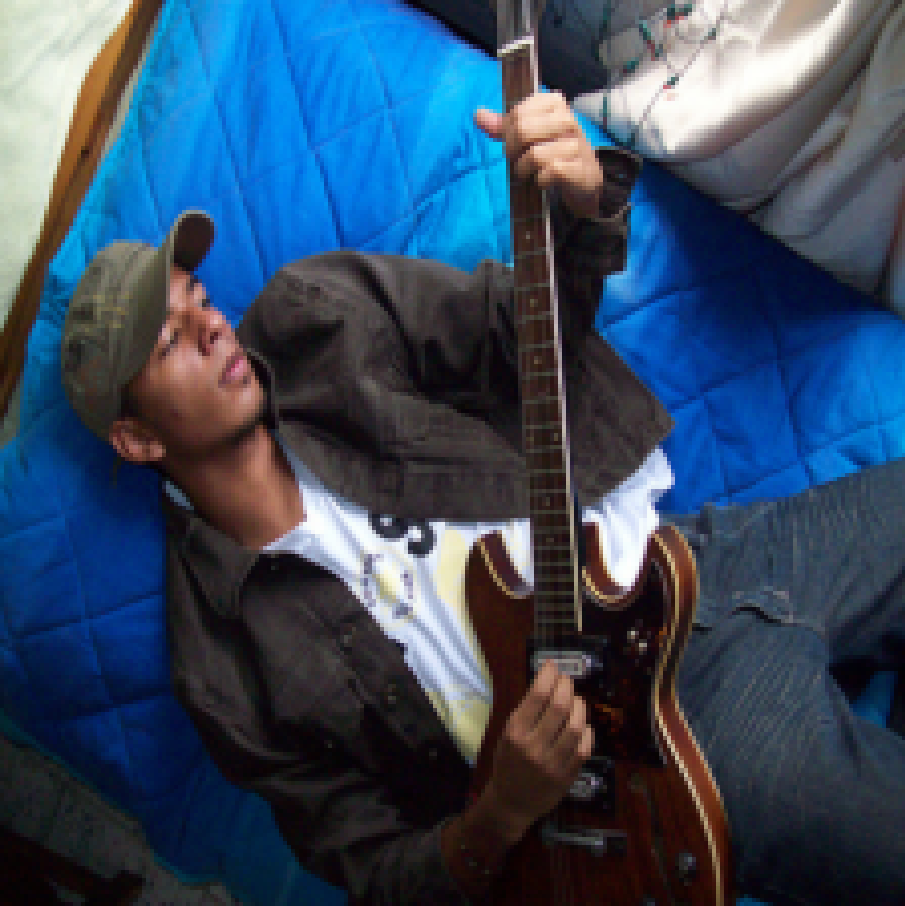} &
\includegraphics[width=0.15\textwidth]{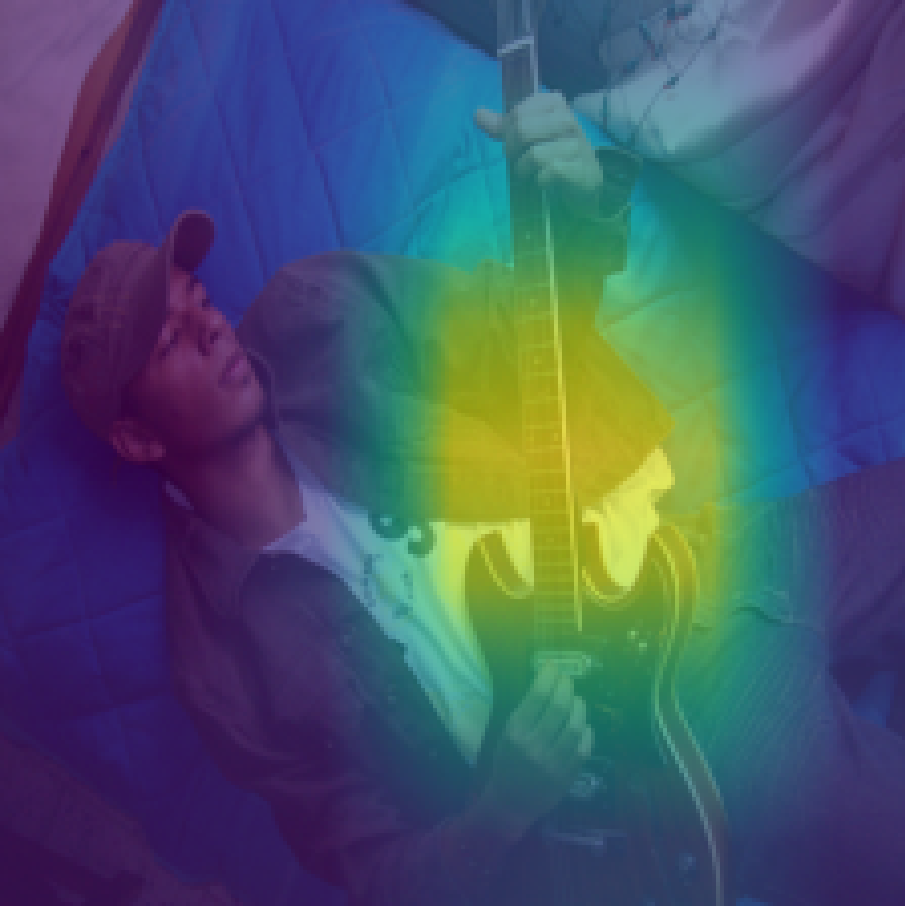} & \includegraphics[width=0.15\textwidth]{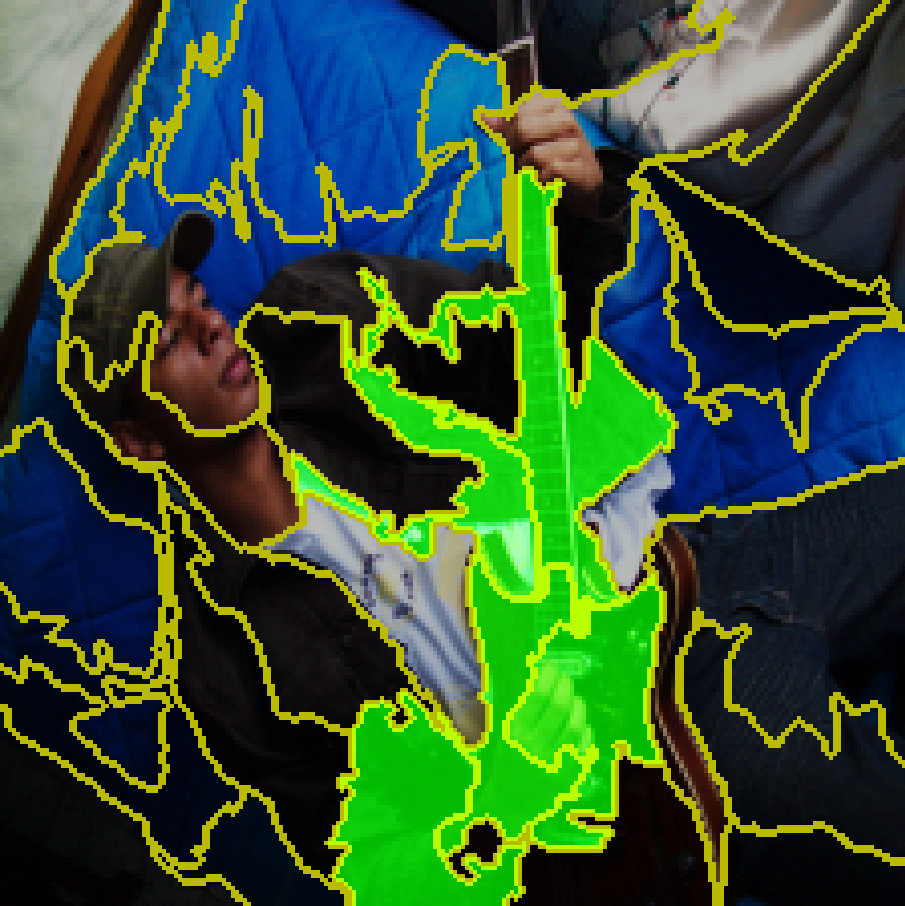}  &
\includegraphics[width=0.15\textwidth]{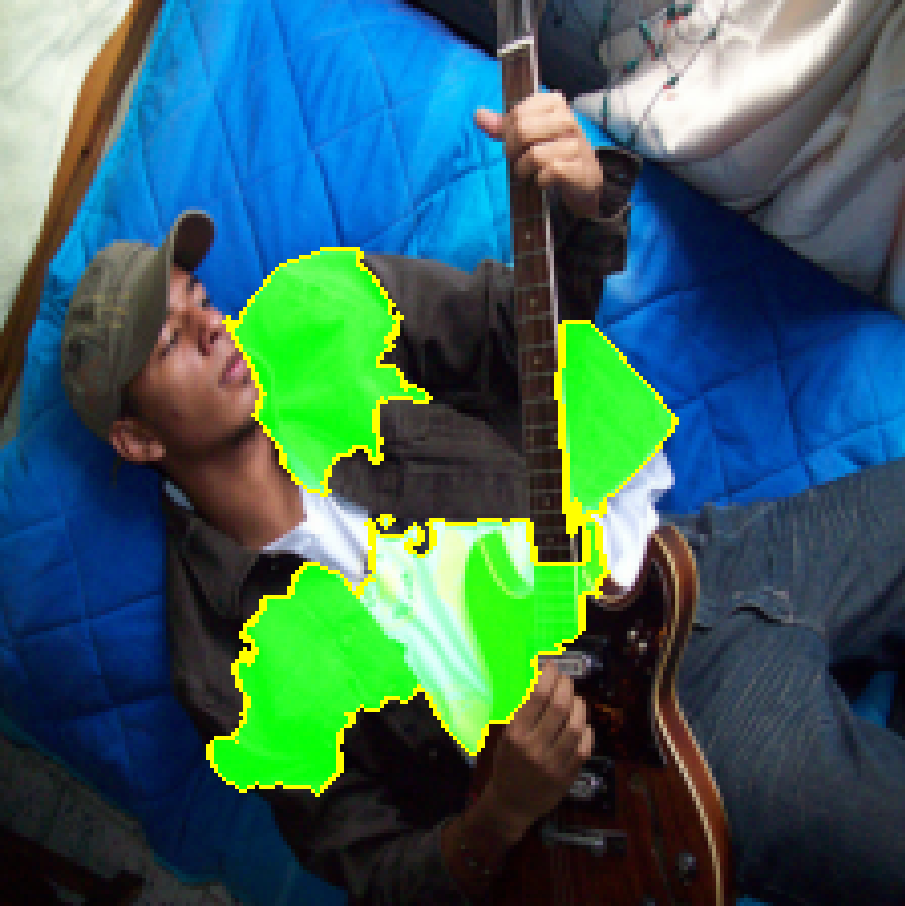} & \includegraphics[width=0.15\textwidth]{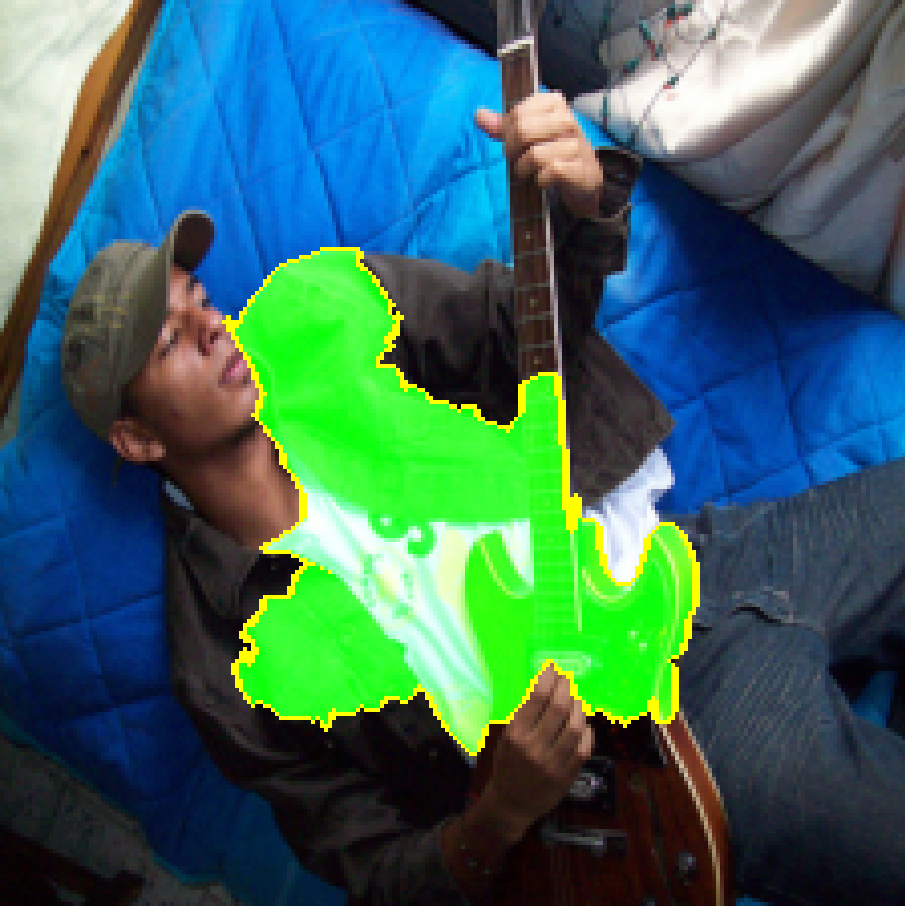}\\
\footnotesize P('Electric Guitar') = 0.33&&&&\\
 \includegraphics[width=0.15\textwidth]{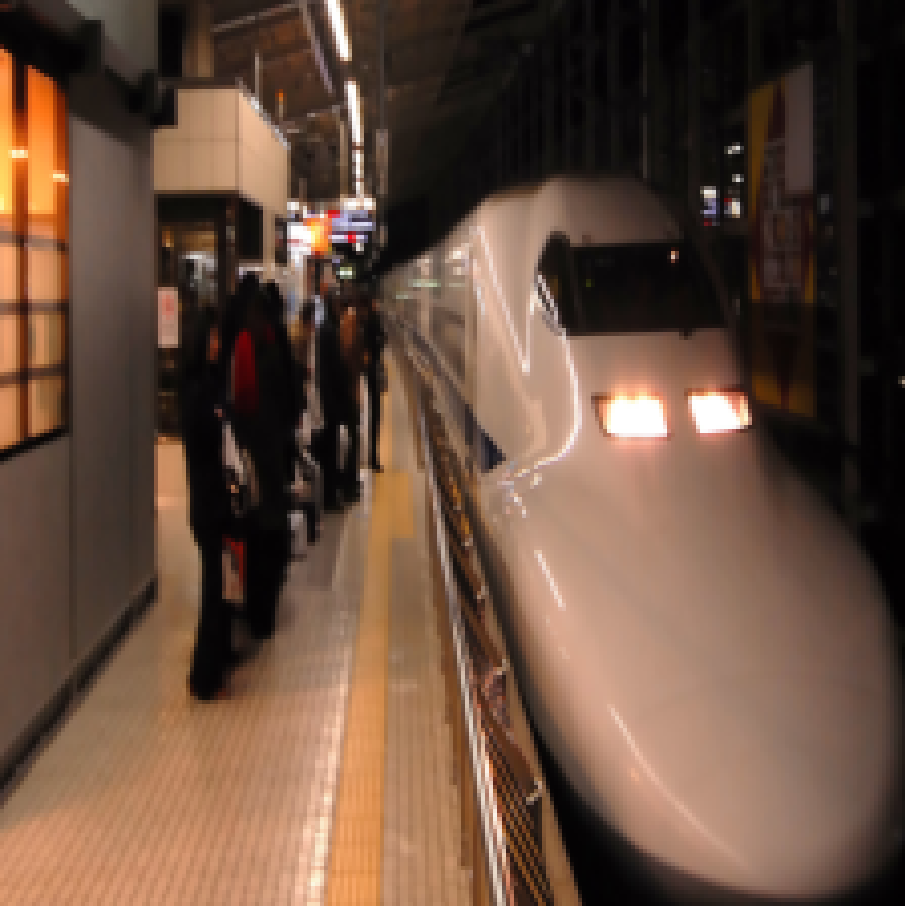} &
\includegraphics[width=0.15\textwidth]{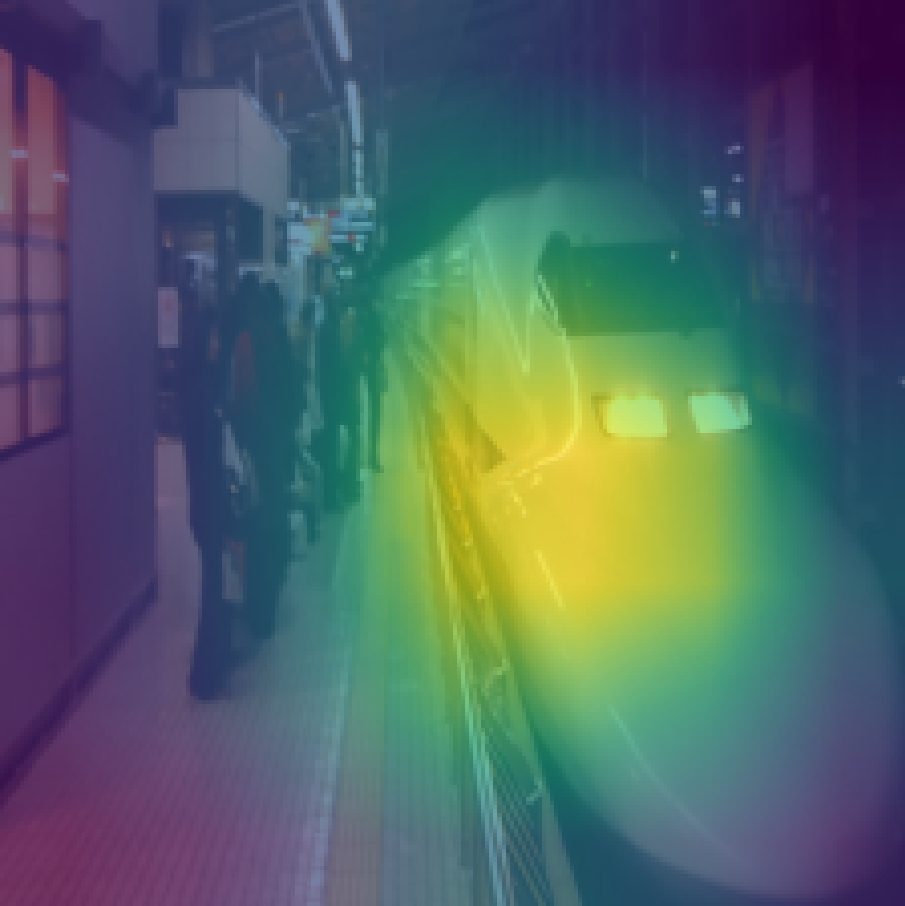} & \includegraphics[width=0.15\textwidth]{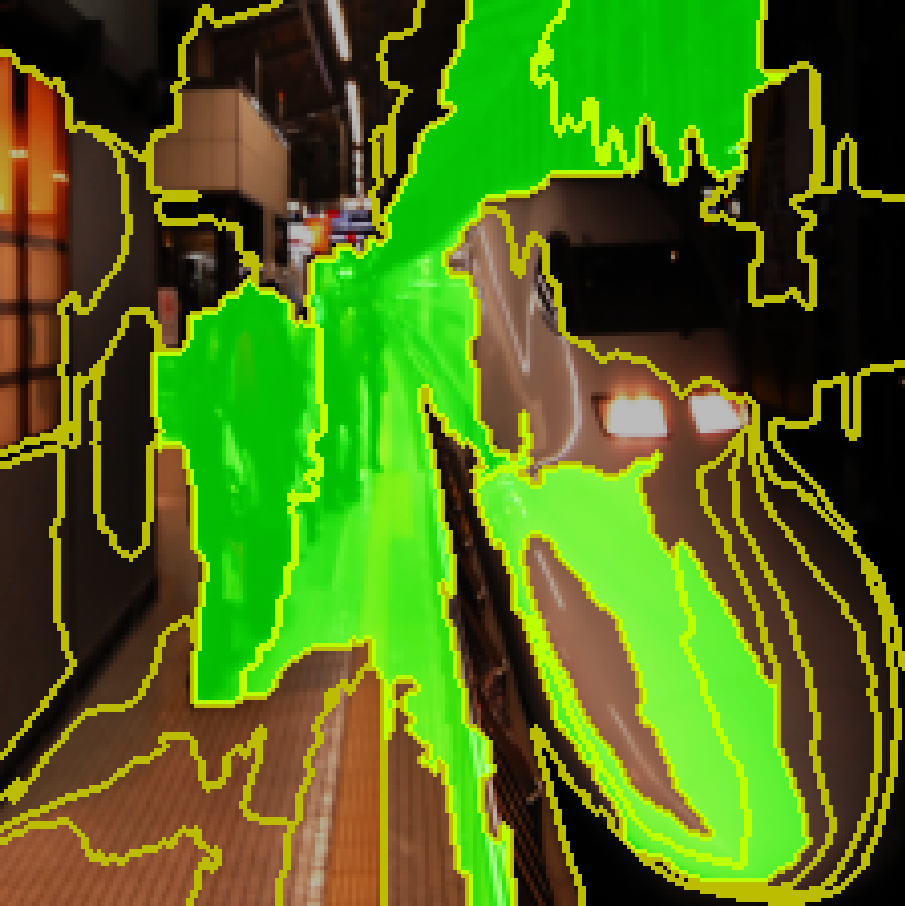}  &
\includegraphics[width=0.15\textwidth]{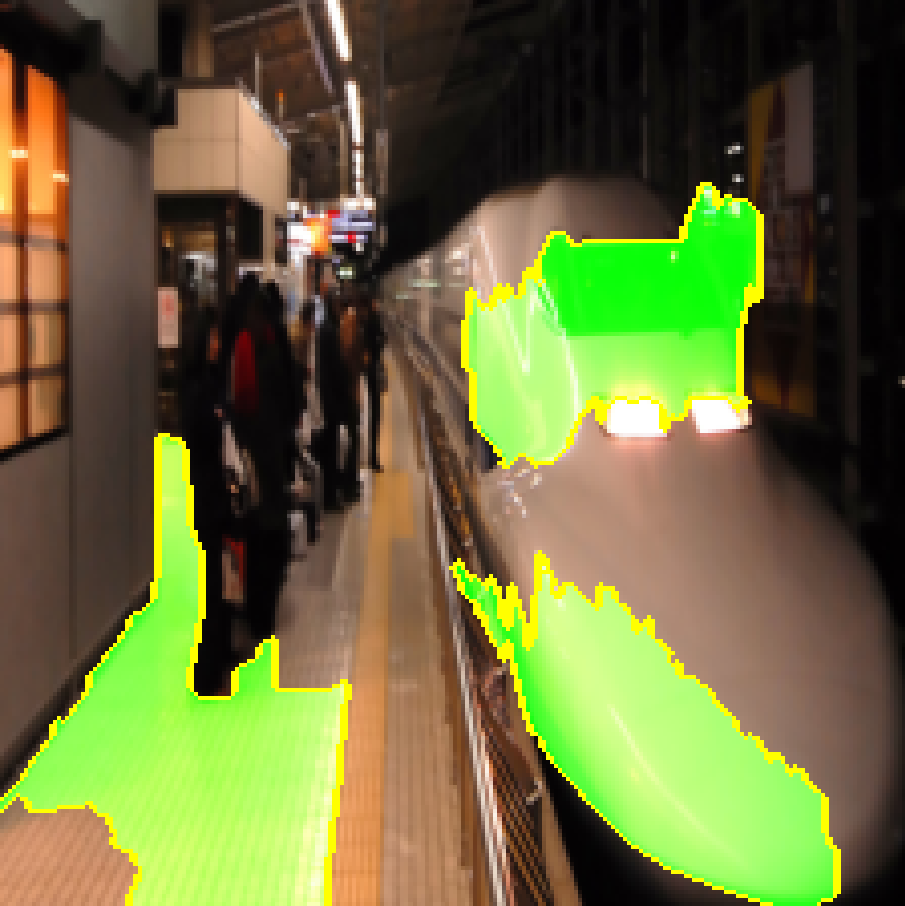} & \includegraphics[width=0.15\textwidth]{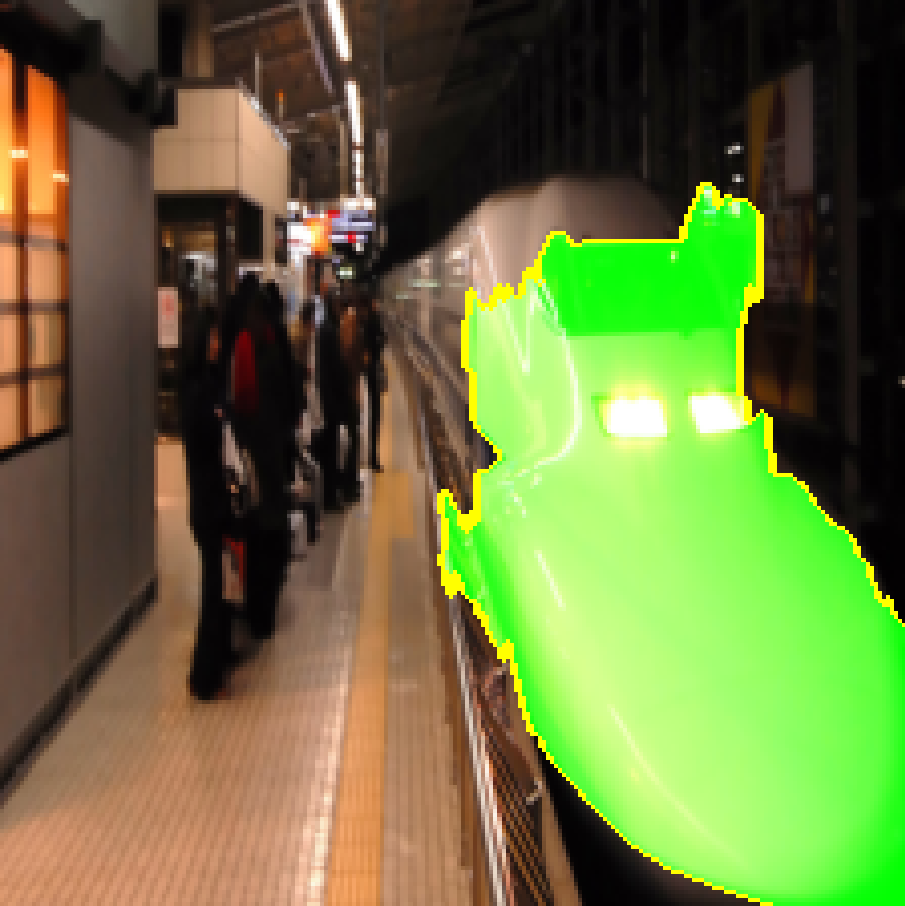}\\
\footnotesize P('Bullet Train') = 0.99&&&&\\
\end{tabular}
\caption{Comparisons of image explanation maps for Resnet-18\cite{resnet} using some samples from the Imagenet dataset\cite{imagenet}}
\label{fig:image}
\end{figure*}
\subsection{Experimental Setup}
We used the open-source GPyOpt\cite{gpyopt} library for UnRAvEL and the Scikit-Learn\citep{sklearn} implementation of the Extra Trees Regressor and Support Vector classifier to simulate the black-box prediction models\footnote{Source code available at \url{https://github.com/adityasaini70/UnRAvEL}} . 
\subsubsection{Baselines}
Since UnRAvEL derives only the standard deviation of each feature from the training data, we intend to compare it with methods employing similar assumptions in their workflow. Going through dense literature of XAI methods, we found that very few perturbation-based models agnostic XAI approaches exist, which do not use the training data as part of their workflow. Techniques like DLIME\cite{DLIME} cluster the entire training data while KernelSHAP requires an explicit passing of the training dataset as a background dataset in the sampling process. We use LIME and BayLIME\cite{BayLIME} as our baselines primarily because they take only the initial sample and variance from the training data. Apart from that, a linear kernel accompanied with a GP model is synonymous with doing Bayesian Linear regression; thus, we argue that UnRAvEL-L is a method synonymous with BayLIME. For the image datasets, we used the popular Grad-CAM\cite{GradCAM} tool to give an idea of the gradient-based explanation maps and plotted the explanation maps for LIME at $100$ samples and $10000$ samples.
\begin{table}
    
\caption{Stability metric computed for 10 randomly selected test samples. The bold values depict the most optimum results.}
  \begin{tabular}{cccccc}
    \toprule
    \textbf{Dataset} & \textbf{LIME} & \textbf{BayLIME} &\textbf{UnRAvEL-L} & \textbf{UnRAvEL}\\
    \midrule
    \textbf{Parkinson's} & 0.743& 0.738&0.499&\textbf{0.146}\\
    \textbf{Cancer} &0.826&0.824&0.655&\textbf{0.295}  \\
    \textbf{Adult} &0.520& 0.524&0.402&\textbf{0.288}  \\
    \textbf{Boston} &0.664&0.668&\textbf{0.462}& 0.539  \\
    \textbf{Bodyfat} &0.687&0.693&\textbf{0.503}& 0.701\\ 
    \bottomrule
  \end{tabular}
  \label{tab:stability_experiments}

\end{table}

\subsection{Evaluation criteria}
\subsubsection{Stability in repeated explanations}
For evaluating the inconsistency in explanations over multiple runs, we ran UnRAvEL, UnRAvEL-L, and the baselines using $100$ surrogate samples and collected $50$ consecutive explanations for $10$ randomly selected index samples for each of the five datasets described in Table\ref{tab:dataset_description}. We used the Jaccard's distance(J) score \cite{DLIME} for the explanation in the $i$-th, and $j$-th run can be computed as follows:
    \begin{equation}
        J(X_i,X_j) = 1-\frac{|X_i \cap X_j|}{|X_i \cup X_j|}
    \end{equation}
where $X_i$ and $X_j$ are sets consisting of top-5 features for iterations $i$ and $j$. Intuitively, it can be observed that $J(X_i,X_j)=0$ if $X_i$ and $X_j$ have the same features, and $J(X_i,X_j)=0$ if they have no common features. Thus, a consistent explainer module will have a relatively lower value of this metric than a relatively inconsistent explainer module. We averaged this metric over all possible combinations of iterations and the $10$ index samples. The results are as given in Table~\ref{tab:stability_experiments} where it can be be seen that for the three classification datasets, UnRAvEL outperforms both LIME and BayLIME. Moreover, for both the regression datasets, UnRAvEL-L, i.e., UnRAvEL with a Linear kernel, outperforms the rest. This shows the superior stability performance of UnRAvEL along with increased flexibility to employ any kernel.

We also plot the mean importance score in Fig.~\ref{fig:mip_comparison} as obtained from the ARD explainer for two samples from the Adult and Boston dataset to understand the type of explanations over which the stability score was computed. The plots were averaged over ten consecutive iterations per index sample.

\subsubsection{UnRAvEL Surrogate Data}
For evaluating the quality of the surrogate dataset generated by UnRAvEL, we collected the surrogate data generated during the stability experiment for two samples from the Adult dataset and used the sparse linear explainer for generating the importance scores. The violin plot for all the features can be found in Fig.~\ref{fig:surrogate_data_comparison}. It can be seen that the importance scores generated by UnRAvEL-LIME (UnRAvEL data coupled with sparse linear explainer) have very low uncertainty as compared to the importance scores generated by LIME. This supports our argument regarding the information gain aspect of the  surrogate data generated by UnRAvEL, using which even sparse linear models can compute highly consistent explanations. 

\subsubsection{UnRAvEL on Image Datasets}
To showcase the UnRAvEL's ability in the case of image datasets, we generated the prediction probabilities using a pre-trained ResNet-18 model\citep{resnet} for sample images from the Imagenet dataset\citep{imagenet}. We used a popular segmentation algorithm called Simple Linear Iterative Clustering (SLIC)\citep{slic} to form $50$ segments of every image and generated a corresponding feature space with continuous attributes where each feature depicted the presence of that index component in the original image. We then used the UnRAvEL classification module on this newly generated feature space to generate the importance score. As a baseline, we plotted the output from Grad-CAM and LIME at $100$ and $10000$ samples. The results are as given in Fig.\ref{fig:image}. UnRAvEL at just $100$ samples can produce explanations that are semantically accurate and are consistent with LIME at $10000$ samples and Grad-CAM. This shows that UnRAvEL is an efficient sample method generating compact yet informative surrogate datasets. 

\section{Conclusions and Future Work}
In this work, we proposed UnRAvEL, a novel explainable AI method that employs a novel acquisition function based on Active Learning, followed by a Gaussian process regression model for obtaining locally interpretable, model agnostic explanations. Through experiments, we showed the stability of UnRAvEL as compared to LIME and BayLIME on tabular structured data. Moreover, we showed the capabilities of UnRAvEL as a surrogate data generation module that can be coupled with custom explainer modules as well. We also showed the sample efficiency of UnRAvEL as compared to LIME on the Imagenet dataset. 

To the best of the authors' knowledge, this is the first work that designs a novel information-gain-based acquisition function, FUR (Faithful Uncertainty Reduction), in the context of Explainable AI and paves the way for many such  exciting directions. The acquisition function used in UnRAvEL allows it to ultimately be independent of the explainer module. This observation can be extended to devise a global explanation module of UnRAvEL, which is a part of our future work. As seen from the experiments, the kernel used in the Gaussian Process explainer can be utilized in many domain-specific applications to make UnRAvEL more robust and stable.
\begin{acks}
This work was supported by the DST funded TiH project and I-HUB Anubhuti Chanakya Fellowship Grant.
\end{acks}

\bibliographystyle{ACM-Reference-Format.bst}
\bibliography{paper.bib}


\begin{thebibliography}{34}


\ifx \showCODEN    \undefined \def \showCODEN     #1{\unskip}     \fi
\ifx \showDOI      \undefined \def \showDOI       #1{#1}\fi
\ifx \showISBNx    \undefined \def \showISBNx     #1{\unskip}     \fi
\ifx \showISBNxiii \undefined \def \showISBNxiii  #1{\unskip}     \fi
\ifx \showISSN     \undefined \def \showISSN      #1{\unskip}     \fi
\ifx \showLCCN     \undefined \def \showLCCN      #1{\unskip}     \fi
\ifx \shownote     \undefined \def \shownote      #1{#1}          \fi
\ifx \showarticletitle \undefined \def \showarticletitle #1{#1}   \fi
\ifx \showURL      \undefined \def \showURL       {\relax}        \fi
\providecommand\bibfield[2]{#2}
\providecommand\bibinfo[2]{#2}
\providecommand\natexlab[1]{#1}
\providecommand\showeprint[2][]{arXiv:#2}

\bibitem[Achanta et~al\mbox{.}(2012)]%
        {slic}
\bibfield{author}{\bibinfo{person}{Radhakrishna Achanta}, \bibinfo{person}{Appu
  Shaji}, \bibinfo{person}{Kevin Smith}, \bibinfo{person}{Aur{\'e}lien Lucchi},
  \bibinfo{person}{Pascal Fua}, {and} \bibinfo{person}{Sabine S{\"u}sstrunk}.}
  \bibinfo{year}{2012}\natexlab{}.
\newblock \showarticletitle{SLIC Superpixels Compared to State-of-the-Art
  Superpixel Methods}.
\newblock \bibinfo{journal}{\emph{IEEE Transactions on Pattern Analysis and
  Machine Intelligence}}  \bibinfo{volume}{34} (\bibinfo{year}{2012}),
  \bibinfo{pages}{2274--2282}.
\newblock


\bibitem[Adadi and Berrada(2018)]%
        {adadi2018peeking}
\bibfield{author}{\bibinfo{person}{Amina Adadi} {and} \bibinfo{person}{Mohammed
  Berrada}.} \bibinfo{year}{2018}\natexlab{}.
\newblock \showarticletitle{Peeking inside the black-box: a survey on
  explainable artificial intelligence (XAI)}.
\newblock \bibinfo{journal}{\emph{IEEE access}}  \bibinfo{volume}{6}
  (\bibinfo{year}{2018}), \bibinfo{pages}{52138--52160}.
\newblock


\bibitem[authors(2016)]%
        {gpyopt}
\bibfield{author}{\bibinfo{person}{The~GPyOpt authors}.}
  \bibinfo{year}{2016}\natexlab{}.
\newblock \bibinfo{title}{GPyOpt: A Bayesian Optimization framework in Python}.
\newblock \bibinfo{howpublished}{\url{http://github.com/SheffieldML/GPyOpt}}.
\newblock


\bibitem[Dombrowski et~al\mbox{.}(2019)]%
        {Fooling_geometry}
\bibfield{author}{\bibinfo{person}{Ann-Kathrin Dombrowski},
  \bibinfo{person}{Maximilian Alber}, \bibinfo{person}{Christopher~J. Anders},
  \bibinfo{person}{Marcel Ackermann}, \bibinfo{person}{Klaus-Robert
  M{\"u}ller}, {and} \bibinfo{person}{Pan Kessel}.}
  \bibinfo{year}{2019}\natexlab{}.
\newblock \showarticletitle{Explanations can be manipulated and geometry is to
  blame}. In \bibinfo{booktitle}{\emph{NeurIPS}}.
\newblock


\bibitem[Dua and Graff(2017)]%
        {UCI}
\bibfield{author}{\bibinfo{person}{Dheeru Dua} {and} \bibinfo{person}{Casey
  Graff}.} \bibinfo{year}{2017}\natexlab{}.
\newblock \bibinfo{title}{{UCI} Machine Learning Repository}.
\newblock
\newblock
\urldef\tempurl%
\url{http://archive.ics.uci.edu/ml}
\showURL{%
\tempurl}


\bibitem[Forrester et~al\mbox{.}(2008)]%
        {forrester}
\bibfield{author}{\bibinfo{person}{Alexander I.~J. Forrester},
  \bibinfo{person}{Andras Sobester}, {and} \bibinfo{person}{Andy~J. Keane}.}
  \bibinfo{year}{2008}\natexlab{}.
\newblock \showarticletitle{Engineering Design via Surrogate Modelling - A
  Practical Guide}.
\newblock


\bibitem[Guo et~al\mbox{.}(2018)]%
        {guo2018explaining}
\bibfield{author}{\bibinfo{person}{Wenbo Guo}, \bibinfo{person}{Sui Huang},
  \bibinfo{person}{Yunzhe Tao}, \bibinfo{person}{Xinyu Xing}, {and}
  \bibinfo{person}{Lin Lin}.} \bibinfo{year}{2018}\natexlab{}.
\newblock \showarticletitle{Explaining Deep Learning Models--A Bayesian
  Non-parametric Approach}.
\newblock \bibinfo{journal}{\emph{Advances in neural information processing
  systems}}  \bibinfo{volume}{31} (\bibinfo{year}{2018}).
\newblock


\bibitem[He et~al\mbox{.}(2016)]%
        {resnet}
\bibfield{author}{\bibinfo{person}{Kaiming He}, \bibinfo{person}{X. Zhang},
  \bibinfo{person}{Shaoqing Ren}, {and} \bibinfo{person}{Jian Sun}.}
  \bibinfo{year}{2016}\natexlab{}.
\newblock \showarticletitle{Deep Residual Learning for Image Recognition}.
\newblock \bibinfo{journal}{\emph{2016 IEEE Conference on Computer Vision and
  Pattern Recognition (CVPR)}} (\bibinfo{year}{2016}),
  \bibinfo{pages}{770--778}.
\newblock


\bibitem[Joel et~al\mbox{.}(2021)]%
        {healthcare_attack}
\bibfield{author}{\bibinfo{person}{Marina~Z. Joel}, \bibinfo{person}{Sachin
  Umrao}, \bibinfo{person}{Enoch Chang}, \bibinfo{person}{Rachel Choi},
  \bibinfo{person}{D.~X. Yang}, \bibinfo{person}{Antonio M.~P. Omuro},
  \bibinfo{person}{Roy~S Herbst}, \bibinfo{person}{Harlan~M. Krumholz}, {and}
  \bibinfo{person}{Sanjay Aneja}.} \bibinfo{year}{2021}\natexlab{}.
\newblock \showarticletitle{Adversarial Attack Vulnerability of Deep Learning
  Models for Oncologic Images}. In \bibinfo{booktitle}{\emph{medRxiv}}.
\newblock


\bibitem[Kim and Choi(2020)]%
        {kim2020local}
\bibfield{author}{\bibinfo{person}{Jungtaek Kim} {and}
  \bibinfo{person}{Seungjin Choi}.} \bibinfo{year}{2020}\natexlab{}.
\newblock \showarticletitle{On local optimizers of acquisition functions in
  bayesian optimization}. In \bibinfo{booktitle}{\emph{Joint European
  Conference on Machine Learning and Knowledge Discovery in Databases}}.
  Springer, \bibinfo{pages}{675--690}.
\newblock


\bibitem[Lewis and Gale(1994)]%
        {UR}
\bibfield{author}{\bibinfo{person}{David~D. Lewis} {and}
  \bibinfo{person}{William~A. Gale}.} \bibinfo{year}{1994}\natexlab{}.
\newblock \showarticletitle{A sequential algorithm for training text
  classifiers}. In \bibinfo{booktitle}{\emph{SIGIR '94}}.
\newblock


\bibitem[Lundberg and Lee(2017)]%
        {SHAP}
\bibfield{author}{\bibinfo{person}{Scott~M Lundberg} {and}
  \bibinfo{person}{Su-In Lee}.} \bibinfo{year}{2017}\natexlab{}.
\newblock \showarticletitle{A Unified Approach to Interpreting Model
  Predictions}.
\newblock \bibinfo{journal}{\emph{NeurIPS}}  \bibinfo{volume}{30}
  (\bibinfo{year}{2017}), \bibinfo{pages}{4765--4774}.
\newblock


\bibitem[Mittelstadt et~al\mbox{.}(2019)]%
        {XXAI}
\bibfield{author}{\bibinfo{person}{Brent~Daniel Mittelstadt},
  \bibinfo{person}{Chris Russell}, {and} \bibinfo{person}{Sandra Wachter}.}
  \bibinfo{year}{2019}\natexlab{}.
\newblock \showarticletitle{Explaining Explanations in AI}.
\newblock \bibinfo{journal}{\emph{Proceedings of the Conference on Fairness,
  Accountability, and Transparency}} (\bibinfo{year}{2019}).
\newblock


\bibitem[Neal(1997)]%
        {NealGP}
\bibfield{author}{\bibinfo{person}{Radford~M Neal}.}
  \bibinfo{year}{1997}\natexlab{}.
\newblock \showarticletitle{Monte Carlo Implementation of Gaussian Process
  Models for Bayesian Regression and Classification}.
\newblock \bibinfo{journal}{\emph{arXiv preprint physics/9701026}}
  (\bibinfo{year}{1997}).
\newblock


\bibitem[Paananen et~al\mbox{.}(2019)]%
        {ARDpaananen}
\bibfield{author}{\bibinfo{person}{T Paananen}, \bibinfo{person}{J Piironen},
  \bibinfo{person}{M Andersen}, {and} \bibinfo{person}{A Vehtari}.}
  \bibinfo{year}{2019}\natexlab{}.
\newblock \showarticletitle{Variable selection for Gaussian processes via
  sensitivity analysis of the posterior predictive distribution}. In
  \bibinfo{booktitle}{\emph{AISTATS}}. PMLR, \bibinfo{pages}{1743--1752}.
\newblock


\bibitem[Pedregosa and et~al.(2011)]%
        {sklearn}
\bibfield{author}{\bibinfo{person}{F. Pedregosa} {and} \bibinfo{person}{et
  al.}} \bibinfo{year}{2011}\natexlab{}.
\newblock \showarticletitle{Scikit-learn: Machine Learning in {P}ython}.
\newblock \bibinfo{journal}{\emph{JMLR}}  \bibinfo{volume}{12}
  (\bibinfo{year}{2011}), \bibinfo{pages}{2825--2830}.
\newblock


\bibitem[Rasmussen(2003)]%
        {GPRasmussen}
\bibfield{author}{\bibinfo{person}{Carl~E Rasmussen}.}
  \bibinfo{year}{2003}\natexlab{}.
\newblock \showarticletitle{Gaussian processes in machine learning}. Springer,
  \bibinfo{pages}{63--71}.
\newblock


\bibitem[Ribeiro et~al\mbox{.}(2016)]%
        {LIME}
\bibfield{author}{\bibinfo{person}{Marco~T Ribeiro}, \bibinfo{person}{S Singh},
  {and} \bibinfo{person}{C Guestrin}.} \bibinfo{year}{2016}\natexlab{}.
\newblock \showarticletitle{Why should i trust you? Explaining the predictions
  of any classifier}. In \bibinfo{booktitle}{\emph{Proceedings of ACM SIGKDD}}.
  \bibinfo{pages}{1135--1144}.
\newblock


\bibitem[Russakovsky et~al\mbox{.}(2015)]%
        {imagenet}
\bibfield{author}{\bibinfo{person}{Olga Russakovsky}, \bibinfo{person}{Jia
  Deng}, \bibinfo{person}{Hao Su}, \bibinfo{person}{Jonathan Krause},
  \bibinfo{person}{Sanjeev Satheesh}, \bibinfo{person}{Sean Ma},
  \bibinfo{person}{Zhiheng Huang}, \bibinfo{person}{Andrej Karpathy},
  \bibinfo{person}{Aditya Khosla}, \bibinfo{person}{Michael Bernstein},
  \bibinfo{person}{Alexander~C. Berg}, {and} \bibinfo{person}{Li Fei-Fei}.}
  \bibinfo{year}{2015}\natexlab{}.
\newblock \showarticletitle{{ImageNet Large Scale Visual Recognition
  Challenge}}.
\newblock \bibinfo{journal}{\emph{International Journal of Computer Vision
  (IJCV)}} \bibinfo{volume}{115}, \bibinfo{number}{3} (\bibinfo{year}{2015}),
  \bibinfo{pages}{211--252}.
\newblock
\urldef\tempurl%
\url{https://doi.org/10.1007/s11263-015-0816-y}
\showDOI{\tempurl}


\bibitem[S et~al\mbox{.}(2020)]%
        {FoolingLIME}
\bibfield{author}{\bibinfo{person}{Dylan S}, \bibinfo{person}{S Hilgard},
  \bibinfo{person}{E Jia}, \bibinfo{person}{S Singh}, {and} \bibinfo{person}{H
  Lakkaraju}.} \bibinfo{year}{2020}\natexlab{}.
\newblock \showarticletitle{Fooling {LIME} and {SHAP}: Adversarial Attacks on
  Post hoc Explanation Methods}.
\newblock \bibinfo{journal}{\emph{AAAI Conference on AI, Ethics, and Society}}
  (\bibinfo{year}{2020}).
\newblock


\bibitem[Saito et~al\mbox{.}(2020)]%
        {SmartSamplingGAN}
\bibfield{author}{\bibinfo{person}{Sean Saito}, \bibinfo{person}{Eugene Chua},
  \bibinfo{person}{Nicholas Capel}, {and} \bibinfo{person}{Rocco Hu}.}
  \bibinfo{year}{2020}\natexlab{}.
\newblock \showarticletitle{Improving LIME Robustness with Smarter Locality
  Sampling}.
\newblock \bibinfo{journal}{\emph{arXiv:2006.12302}} (\bibinfo{year}{2020}).
\newblock


\bibitem[Selvaraju et~al\mbox{.}(2019)]%
        {GradCAM}
\bibfield{author}{\bibinfo{person}{Ramprasaath~R. Selvaraju},
  \bibinfo{person}{Abhishek Das}, \bibinfo{person}{Ramakrishna Vedantam},
  \bibinfo{person}{Michael Cogswell}, \bibinfo{person}{Devi Parikh}, {and}
  \bibinfo{person}{Dhruv Batra}.} \bibinfo{year}{2019}\natexlab{}.
\newblock \showarticletitle{Grad-CAM: Visual Explanations from Deep Networks
  via Gradient-Based Localization}.
\newblock \bibinfo{journal}{\emph{International Journal of Computer Vision}}
  \bibinfo{volume}{128} (\bibinfo{year}{2019}), \bibinfo{pages}{336--359}.
\newblock


\bibitem[Shahriari et~al\mbox{.}(2015)]%
        {BORef1}
\bibfield{author}{\bibinfo{person}{B Shahriari}, \bibinfo{person}{K Swersky},
  \bibinfo{person}{Z Wang}, \bibinfo{person}{Ryan~P Adams}, {and}
  \bibinfo{person}{N De~Freitas}.} \bibinfo{year}{2015}\natexlab{}.
\newblock \showarticletitle{Taking the human out of the loop: A review of
  Bayesian optimization}.
\newblock \bibinfo{journal}{\emph{Proc. of the IEEE}} \bibinfo{volume}{104},
  \bibinfo{number}{1} (\bibinfo{year}{2015}), \bibinfo{pages}{148--175}.
\newblock


\bibitem[Shankaranarayana and Runje(2019)]%
        {ALIME}
\bibfield{author}{\bibinfo{person}{S~M Shankaranarayana} {and}
  \bibinfo{person}{D Runje}.} \bibinfo{year}{2019}\natexlab{}.
\newblock \showarticletitle{ALIME: Autoencoder based approach for local
  interpretability}. In \bibinfo{booktitle}{\emph{IDEAL}}. Springer,
  \bibinfo{pages}{454--463}.
\newblock


\bibitem[Shrikumar et~al\mbox{.}(2017)]%
        {Deeplift}
\bibfield{author}{\bibinfo{person}{A Shrikumar}, \bibinfo{person}{P Greenside},
  {and} \bibinfo{person}{A Kundaje}.} \bibinfo{year}{2017}\natexlab{}.
\newblock \showarticletitle{Learning important features through propagating
  activation differences}. In \bibinfo{booktitle}{\emph{Proceedings of ICML}}.
  PMLR, \bibinfo{pages}{3145--3153}.
\newblock


\bibitem[Slack et~al\mbox{.}(2020)]%
        {Slack2020ReliablePH}
\bibfield{author}{\bibinfo{person}{Dylan Slack}, \bibinfo{person}{Sophie
  Hilgard}, \bibinfo{person}{Sameer Singh}, {and} \bibinfo{person}{Himabindu
  Lakkaraju}.} \bibinfo{year}{2020}\natexlab{}.
\newblock \showarticletitle{Reliable Post hoc Explanations: Modeling
  Uncertainty in Explainability}.
\newblock


\bibitem[Srinivas et~al\mbox{.}(2010a)]%
        {srinivasLCB}
\bibfield{author}{\bibinfo{person}{Niranjan Srinivas}, \bibinfo{person}{Andreas
  Krause}, \bibinfo{person}{Sham Kakade}, {and} \bibinfo{person}{Matthias
  Seeger}.} \bibinfo{year}{2010}\natexlab{a}.
\newblock \showarticletitle{Gaussian process optimization in the bandit
  setting: no regret and experimental design}. In
  \bibinfo{booktitle}{\emph{Proceedings of the 27th International Conference on
  International Conference on Machine Learning}}. \bibinfo{pages}{1015--1022}.
\newblock


\bibitem[Srinivas et~al\mbox{.}(2010b)]%
        {UCBsrinivas2010gaussian}
\bibfield{author}{\bibinfo{person}{Niranjan Srinivas}, \bibinfo{person}{Andreas
  Krause}, \bibinfo{person}{Sham Kakade}, {and} \bibinfo{person}{Matthias~W
  Seeger}.} \bibinfo{year}{2010}\natexlab{b}.
\newblock \showarticletitle{Gaussian Process Optimization in the Bandit
  Setting: No Regret and Experimental Design}. In
  \bibinfo{booktitle}{\emph{ICML}}.
\newblock


\bibitem[Su et~al\mbox{.}(2019)]%
        {OnepixelFooling}
\bibfield{author}{\bibinfo{person}{Jiawei Su},
  \bibinfo{person}{Danilo~Vasconcellos Vargas}, {and} \bibinfo{person}{Kouichi
  Sakurai}.} \bibinfo{year}{2019}\natexlab{}.
\newblock \showarticletitle{One pixel attack for fooling deep neural networks}.
\newblock \bibinfo{journal}{\emph{IEEE Transactions on Evolutionary
  Computation}} \bibinfo{volume}{23}, \bibinfo{number}{5}
  (\bibinfo{year}{2019}), \bibinfo{pages}{828--841}.
\newblock


\bibitem[V et~al\mbox{.}(2020)]%
        {Indices}
\bibfield{author}{\bibinfo{person}{Giorgio V}, \bibinfo{person}{Enrico B},
  \bibinfo{person}{F. Chesani}, \bibinfo{person}{Alessandro P}, {and}
  \bibinfo{person}{D. Capuzzo}.} \bibinfo{year}{2020}\natexlab{}.
\newblock \showarticletitle{Statistical stability indices for {LIME}: obtaining
  reliable explanations for Machine Learning models}.
\newblock   \bibinfo{volume}{abs/2001.11757} (\bibinfo{year}{2020}).
\newblock


\bibitem[Visani et~al\mbox{.}(2020)]%
        {OptiLIME}
\bibfield{author}{\bibinfo{person}{G. Visani}, \bibinfo{person}{E. Bagli},
  {and} \bibinfo{person}{F. Chesani}.} \bibinfo{year}{2020}\natexlab{}.
\newblock \showarticletitle{{OptiLIME}: Optimized LIME Explanations for
  Diagnostic Computer Algorithms}.
\newblock \bibinfo{journal}{\emph{arXiv:2006.05714}} (\bibinfo{year}{2020}).
\newblock


\bibitem[Zafar and Khan(2019)]%
        {DLIME}
\bibfield{author}{\bibinfo{person}{Muhammad~R. Zafar} {and}
  \bibinfo{person}{N.~M. Khan}.} \bibinfo{year}{2019}\natexlab{}.
\newblock \showarticletitle{{DLIME}: a deterministic local interpretable
  model-agnostic explanations approach for computer-aided diagnosis systems}.
\newblock \bibinfo{journal}{\emph{arXiv:1906.10263}} (\bibinfo{year}{2019}).
\newblock


\bibitem[Zhao et~al\mbox{.}(2020)]%
        {BayLIME}
\bibfield{author}{\bibinfo{person}{Xingyu Zhao}, \bibinfo{person}{Xiaowei
  Huang}, \bibinfo{person}{Valentin Robu}, {and} \bibinfo{person}{David
  Flynn}.} \bibinfo{year}{2020}\natexlab{}.
\newblock \showarticletitle{{BayLIME}: Bayesian Local Interpretable
  Model-Agnostic Explanations}.
\newblock \bibinfo{journal}{\emph{Proc. of UAI}} (\bibinfo{year}{2020}).
\newblock


\bibitem[Zhou et~al\mbox{.}(2019)]%
        {social_attack}
\bibfield{author}{\bibinfo{person}{Wei Zhou}, \bibinfo{person}{Xiao-Fang Yuan},
  \bibinfo{person}{Wenjun Chai}, {and} \bibinfo{person}{Hui Ma}.}
  \bibinfo{year}{2019}\natexlab{}.
\newblock \showarticletitle{Deep Learning Based Attack On Social Authentication
  System}.
\newblock \bibinfo{journal}{\emph{2019 IEEE 3rd Information Technology,
  Networking, Electronic and Automation Control Conference (ITNEC)}}
  (\bibinfo{year}{2019}), \bibinfo{pages}{982--986}.
\newblock


\end{thebibliography}

\newpage
\appendix
\section{Lemma 1}
\begin{lemma}
Let $\mathcal{X} \subset \mathbb{R}^d$ be a compact space where $\gamma = \max_{\vecx_1,\vecx_2 \in \mathcal{X}} ||\vecx_1 - \vecx_2||_2$. Then,for $\gamma > \epsilon_1 >0$, we have
$\Pr||\vecx_{n,g} - \vecx_{n,l}||_2 \geq \epsilon_1) \leq \frac{d^n_{0,g}}{\epsilon_1} + \frac{(1-\beta_0)\gamma}{\epsilon_1}$ where $\beta_0$ is the probability that some local optimizer is coincided with the index sample $\vecx_0$ and $d^n_{0,g} = \Ex[||\vecx_{n,g} - \vecx_0||_2]$.
\end{lemma}
\begin{proof}
Using Markov inequality for $\epsilon_1$, 
\begin{align}
    &\Pr||\vecx_{n,g} - \vecx_{n,l}||_2 \geq \epsilon_1) \leq \frac{1}{\epsilon_1}\Ex[||\vecx_{n,g} - \vecx_{n,l}||_2]\nonumber\\
    &= \frac{1}{\epsilon_1}\Ex[||\vecx_{n,g} - \vecx_0 + \vecx_0 +  \vecx_{n,l}||_2]\nonumber\\
    &\leq \frac{1}{\epsilon_1}\Ex[||\vecx_{n,g} - \vecx_0||_2] + \frac{1}{\epsilon_1}\Ex[||\vecx_{0} - \vecx_{n,l}||_2]
\end{align}
where the last inequality is obtained using the Minkowski’s inequality. Here,  $d^n_{0,g} = \Ex||\vecx_0 - \vecx_{n,g}||_2$ does not depend on the local optimizer,and we assume it to be a constant. Let $\beta_0$ which represents the probability that the local optimizer is coincided with $\vecx_0$. Using lemmas~1-6 in \cite{kim2020local} we have,
\begin{align}
    &\Pr||\vecx_{n,g} - \vecx_{n,l}||_2 \geq \epsilon_1) \leq \frac{d^n_{0,g}}{\epsilon_1} + \frac{\beta_0}{\epsilon_1}{||\vecx_0 - \vecx_{n,l}||}_{\vecx_{n,l} = \vecx_0} \nonumber\\
    &+ \frac{1-\beta_0}{\epsilon_1}{||\vecx_0 - \vecx_{n,l}||}_{\vecx_{n,l} \neq \vecx_0}\nonumber\\
    & \leq \frac{d^n_{0,g}}{\epsilon_1} + \frac{(1-\beta_0)\gamma}{\epsilon_1}
\end{align}
\end{proof}

For M-Lipshitz continuous functions, we have (Lemma~8 of \cite{kim2020local})
\begin{align}
    \Pr\left( \frac{|f_e(\vecx_0) -f_e(\vecx_{n,l})|}{||\vecx_0 - \vecx_{n,l}||_2}\geq \epsilon_2\right)\leq \frac{M}{\epsilon_2},
\end{align}
where $M$ is the Lipschitz constant of function $f_e$, $f_e$ is M-Lipschitz continuous and $\vecx_{n,l} , \vecx_{0} \in \mathcal{X}$.

\section{Proof of Theorem~1:}
\begin{proof}
First, we write
\begin{align}
    \Pr(|r_{n,g} - r_{n,l}<\epsilon_l) = \Pr(|f_e(\vecx_{n,g}) - f_e(\vecx_{n,l})|<\epsilon_l).
    \label{eq:1Proof}
\end{align}
Here, the event $(|f_e(\vecx_{n,g}) - f_e(\vecx_{n,l})|<\epsilon_l)$ can be rewritten as $(|f_e(\vecx_{n,g}) - f_e(\vecx_0) + f_e(\vecx_0) -  f_e(\vecx_{n,l})|<\epsilon_l)$. Let $\delta_n = |f_e(\vecx_{n,g}) - f_e(\vecx_0)|$, which is independent of the local optimizer. We rewrite 
\eqref{eq:1Proof} as
\begin{align}
    &\Pr(|f_e(\vecx_{n,g}) - f_e(\vecx_{n,l})|<\epsilon_l)\nonumber\\
    &= \Pr(-\epsilon_l-\delta_n < (f_e(\vecx_{n,g}) - f_e(\vecx_{n,l})) < \epsilon_l - \delta_n)\nonumber\\ 
    &\leq  \Pr(|f_e(\vecx_{0}) - f_e(\vecx_{n,l})|<\epsilon_l+\delta_n),
\end{align}
where we obtain the final inequality by  comparing the probabilities on events $-\epsilon_l-\delta_n\leq  f_e(\vecx_{n,g}) - f_e(\vecx_{n,l}) \leq \epsilon_l - \delta_n$ and $|f_e(\vecx_{0}) - f_e(\vecx_{n,l})|<\epsilon_l+\delta_n$. Note that
\begin{align}
    &\Pr(|f_e(\vecx_{0}) - f_e(\vecx_{n,l})|<\epsilon_l+\delta_n) \nonumber\\
    &= \Pr(||\vecx_0 - \vecx_{n,l}||_2.\frac{|f_e(\vecx_{0}) - f_e(\vecx_{n,l})|}{||\vecx_0 - \vecx_{n,l}||_2}<\epsilon_l+\delta_n).
\end{align}
We define two events:
\begin{align}
    E_1 &= (||\vecx_0 - \vecx_{n,l}||_2 < \eta_1)\nonumber\\
    E_2 &= (\frac{|f_e(\vecx_{0}) - f_e(\vecx_{n,l})|}{||\vecx_0 - \vecx_{n,l}||_2} < \eta_2),
\end{align}
where $\eta_1\eta_2 = \epsilon_l+\delta_n$. Since $P(E_1 \cap E_2) \geq 1 - P(E_1^C) - P(E_2^C)$, where $E_1^C$ represents the complement of the event. Hence, 
\begin{align}
   &\Pr(|r_{n,g} - r_{n,l}<\epsilon_l) \geq  \Pr(E_1 \cap E_2) \nonumber\\
   &\geq 1 - \Pr(||\vecx_0 - \vecx_{n,l}||_2 > \eta_1) - \Pr(\frac{|f_e(\vecx_{0}) - f_e(\vecx_{n,l})|}{||\vecx_0 - \vecx_{n,l}||_2} > \eta_2) \nonumber\\
   &\geq 1 - \frac{d^n_{0,g}}{\eta_1} - \frac{(1-\beta_0)\gamma}{\eta_1} - \frac{M}{\eta_2}.
\end{align}
\end{proof}

\end{document}